\documentclass[twoside,11pt,preprint]{article}
\usepackage{jmlr2e}
\usepackage{global_settings}
\usepackage{footnote}
\makesavenoteenv{tabular}
\ShortHeadings{Privacy in Multi-armed Bandits}{Basu, Dimitrakakis and Tossou}
\firstpageno{1}

\begin{document}
\title{Privacy in Multi-armed Bandits:\\ Fundamental Definitions and Lower Bounds}
\author{Debabrota Basu, Christos Dimitrakakis, \and Aristide Y. Tossou\\
\addr{Chalmers University of Technology, Sweden}}

\maketitle

\begin{abstract}
Based on differential privacy (DP) framework, we introduce and unify privacy definitions for the multi-armed bandit algorithms. We represent the framework with a unified graphical model and use it to connect privacy definitions. We derive and contrast lower bounds on the regret of bandit algorithms satisfying these definitions. We leverage a unified proving technique to achieve all the lower bounds. We show that for all of them, the learner's regret is increased by a multiplicative factor dependent on the privacy level $\epsilon$. We observe that the dependency is weaker when we do not require local differential privacy for the rewards. 
\end{abstract}


\section{Introduction}
\emph{Differential Privacy} (DP) is a rigorous and highly successful definition of algorithmic privacy introduced by~\citet{dwork2006calibrating}. Given a definition of neighbourhood for inputs to an algorithm, it defines a notion of privacy that ensures that the algorithm's output renders neighbouring inputs indistinguishable.
Informally, if an algorithm is $\epsilon$-DP, then the amount of information inferred by an adversary about the algorithm's input is bounded by $\epsilon$. 
\begin{definition}[$\epsilon$-Differential Privacy]\label{def:GDP}
	A privacy-preserving mechanism $\mech$ composed with a randomised algorithm $\pol : \mathcal{D} \to \mathcal{A}$ is $\epsilon$-differentially private if for all
	inputs $x, x' \in \mathcal{D}$ with $\lp{x}{x'}{H}= 1$:
	$$\pol(A \mid x) \leq \pol(A \mid x') e^\epsilon, \qquad A \subset \mathcal{A}.$$
\end{definition}
Our notation suggests a view of the algorithm $\pol$  as a conditional distribution. Typically, the inputs to the algorithm are structured so that $x = (x_1, \ldots, x_n)$ corresponds to the data of $n$ individuals, with $x_i$ being the data of the $i$-th individual. $\lp{x}{x'}{H} = 1$ if $x_i \neq x'_i$ for an $i \in [n]$.

However, in many online learning settings, the algorithm might not see all of the data before making a decision, and in partial monitoring problems, where the algorithm might be unable to see all of the data before making a decision. In these cases, Definition~\ref{def:GDP} is not trivial to apply. Thus, different works analysing the privacy of algorithms in these domains adopt subtly different privacy definitions. Our aim is to fix precise definitions and to provide lower bounds on learning while satisfying privacy under each definition.

We consider the \emph{multi-armed bandit} setting~\citep{bellman1956problem,lattimore2018bandit} in particular. This involves an algorithm $\pol$ sequentially choosing among $K$ different arms to maximise expected cumulative reward. At time $t$, the algorithm  \emph{draws an arm} $a_t \in \Act$ and the arms \emph{generate outcomes} $\bx_t = [x_{t,j}]_{j=1}^K \in \real^K$. Then the learner \emph{obtains a reward} $r_t = x_{t,a_t}$ which is equal to the outcome of the selected arm\footnote{More generally, we can consider a reward function $r_t = u(\bx_t, a_t)$.} without observing the other arms. As a motivating example, consider the case of patients arriving sequentially and the doctor assigning different treatments $a_t$. The doctor only sees the effect of the selected treatment, and there is a need to maintain patient's confidentiality.

The natural way to conceive differential privacy is to perceive the algorithm as a Turing machine with input tape $\bx = \bx_1, \bx_2, \ldots$, output tape $a = a_1, a_2, \ldots$, and a random input tape $\omega = \omega_1, \omega_2$ in order to admit stochastic decision making\footnote{We often denote a sequence $[s_1, \ldots, s_t, \ldots, s_T]$ as $s^T$ and $[s_1, \ldots, s'_t, \ldots, s_T]$ as $s'^T$.}. A bandit algorithm $\pol$ will satisfy differential privacy if, for any decision horizon $T$ and privacy level $\epsilon > 0$
\begin{align*}
  \pol(a_1, \ldots, a_t \mid \bx_1, \ldots, \bx_k, \ldots, \bx_T)
 \leq &\pol(a_1, \ldots, a_t \mid \bx_1, \ldots, \bx'_k, \ldots,  \bx_T) e^\epsilon.
\end{align*}
One may reasonably ask why should we not consider a bandit algorithm differentially private only with respect to the observed rewards, i.e.
\begin{align*}
  \pol(a_1, \ldots, a_t \mid r_1, \ldots, r_k, \ldots, r_T)
	\leq &\pol(a_1, \ldots, a_t \mid r_1, \ldots, r'_k, \ldots,  r_T) e^\epsilon.
\end{align*}
As we show later, the two definitions are equivalent in the bandit setting. However, the input to the algorithm is the outcomes generated by the arms, rather than the observed rewards, and since we cannot know a priori whether an arm will be chosen at a particular time. Similarly, we cannot know which sections of the input tape a Turing machine will see. Thus, the generated outcomes are a more suitable definition for differential privacy. 

This brings us to the question: what are the mechanisms to use to guarantee privacy? One popular approach for achieving DP is the \emph{local privacy}~\cite{duchi2013local}. This creates a differentially private version $\by$ of the input $\bx$, such that $\prob(\by \mid \bx) \leq \prob(\by \mid \bx') e^\epsilon$.
The private rewards $\by$ is then used by the algorithm $\pi$ to make decisions. 
In our paper, we examine the efficiency loss i.e. regret relative to algorithms both for the local and non-local mechanisms.

\paragraph{Our Contribution.} 
We define, discuss, and unify different notions of differential privacy (Section~\ref{sec:dp_mab}).  In
particular, we examine the effect of considering different notions of private input, observable output and neighbourhoods on the regret of multi-armed bandits that are operating under a differential privacy
constraint. 
We also discuss mechanisms to achieve differential privacy: both local and non-local (Section~\ref{sec:dp_mech}). We illustrate the differences between those definitions using graphical models. 

We provide a unified framework to prove minimax lower bounds on the regret of both differentially private multi-armed bandits (Section~\ref{sec:theory}).
This is based on a generalised KL-divergence decomposition lemma adapted for local and standard differential privacy definitions. Though the literature consists of problem-dependent regret bounds, these are the first minimax and Bayesian regret bounds for both differentially private bandits. We show that both in general and when differential privacy is achieved using a local mechanism, the regret scales as a multiplicative factor of $\epsilon$. As expected, local privacy mechanisms have a slightly worse performance. 
The regret lower bounds that we prove and the corresponding privacy definitions are summarised in Table~\ref{tab:contributions}.

In Section~\ref{sec:multi-armed-bandit}, we provide the required technicalities of multi-armed bandit problems. In Section~\ref{sec:lbounds}, we elaborate that the proposed lower bounds pose several open problems of designing optimal bandit algorithms that satisfy different notions of privacy. 

Appendix~\ref{sec:proofs_def} and ~\ref{sec:proofs_lb} contain detailed proofs.

\begin{table*}[t!]
\centering
\scalebox{0.7}{\begin{tabular}{c|c|c|c|c}
		\hline 
		& Definition & Minimax  & Bayesian Minimax & Problem-dependent \\
		& for Bandits & Regret &  Regret &  Regret \\  
		\hline 
		& & & & \\
		DP & $\ln \abs{\dfrac{\prob_\pol(a^T \mid \bx^T)}{\prob_\pol(a^T \mid \bx'^T)}} \leq \epsilon$, given $T$ & $\sqrt{\dfrac{(K-1)T \ln(\epsilon^2+1) }{e^{6\epsilon} \epsilon^{(1+\frac{1}{\epsilon})}(\epsilon+B)^{\frac{1}{\epsilon}}}}$ & $\sqrt{\dfrac{(K-1)T \ln(\epsilon^2+1) }{e^{6\epsilon} \epsilon^{(1+\frac{1}{\epsilon})}(\epsilon+B)^{\frac{1}{\epsilon}}}}$ & Open Problem \\ 
		& & & & \\
		Instantaneous DP & $\ln \abs{\dfrac{\prob_\pol(a_t \mid a^{t-1}, \bx^{t-1})}{\prob_\pol(a_t \mid a^{t-1}, \bx'^{t-1})}} \leq \epsilon, \forall t \leq T$ & $\sqrt{\dfrac{(K-1)T}{2\epsilon(e^{2\epsilon}-1)}}$ & $\sqrt{\dfrac{(K-1)T}{2\epsilon(e^{2\epsilon}-1)}}$ & Open Problem  \\ 
		& & & & \\
		Local DP & $\ln \abs{\dfrac{\prob(\bz^{t}\mid r^{t})}{\prob(\bz^{t} \mid r'^{t})}} \leq \epsilon, \forall t \leq T$  & $\dfrac{\sqrt{(K-1)T}}{\min\lbrace 2,e^{\epsilon}\rbrace(e^{\epsilon}-1)}$ & $\dfrac{\sqrt{(K-1)T}}{\min\lbrace 2,e^{\epsilon}\rbrace(e^{\epsilon}-1)}$ & $\dfrac{c(\model) \log T}{2\min\lbrace4,e^{2\epsilon}\rbrace(e^{\epsilon}-1)^2}$   \\ 
		\hline 
	\end{tabular} }
	\caption{Privacy definitions for multi-armed bandits and corresponding regret lower bounds}\label{tab:contributions}
\end{table*}

\section{Multi-armed Bandits}\label{sec:multi-armed-bandit}
The stochastic $K$-armed bandit problem~\citep{bellman1956problem,lattimore2018bandit} involves a learner sequentially choosing among $K$ different arms. At time $t$, the algorithm  \emph{draws an arm} $a_t  \in \Act$ and the arms \emph{generate outcomes} $\bx_t = [x_{t,j}]_{j=1}^K \in \real^K$. Then the learner \emph{obtains reward} $r_t = x_{t,a_t}$, without observing the other arms.

In the \emph{stochastic} setting, the learner is acting within an environment $\model$, with each action generating outcomes with distribution $f_a$ and means $\mu_a \defn \expect(f_a)$, and optimal expected reward $\mu^* \defn \max_a \mu_a$ so that the reward distribution for each arm is $\prob_\model(X_t = r_{t,a}) = f_a$ for all $t$. 

The learner's policy $\pol$ for selecting actions is generally a stochastic mapping $\pol : \Hist \to \dist(\Act)$. Here $\Hist$ is the \emph{observed history}, i.e. the sequence of actions taken and rewards obtained by the learner. 
\ifdefined \workshop
\
\else
The objective is to \emph{maximise the expected cumulative reward}, 
\begin{align}
S(\pol, \model, T) \defn 
\sum_{t=1}^T \mathbb{E}_{\model \pol} [r_t]
= \sum_{a=1}^K \mathbb{E}_{\model \pol} \left[N_a(T)\right]\mu_a,
\end{align}
where $N_a(T)$ denotes the number of time arm $a$ is pulled till time step $T$ and $\mu_a$ is the expected reward of the arm $a$.
\fi

The quality of a learning algorithm $\pol$ is best summarised by its \emph{(expected cumulative) regret}, i.e. its loss in total reward relative to an oracle that knows $\model$:
\begin{equation}
\reg(\pol, \model, T) \defn 
\sum_{a=1}^K \mathbb{E}_{\model \pol}\left[N_a(T)\right](\mu^* - \mu_a).\vspace*{-.5em}
\end{equation}
$N_a(T)$ denotes the number of time arm $a$ is pulled till time step $T$.
This is the cost incurred by the algorithm due to the incomplete information, as it has to play the suboptimal arms to gain information about the suboptimal arms. This process decreases the uncertainty in decision making and facilitates maximisation of the expected cumulative reward.

Lower bounds illustrate the inherent hardness of bandit problems.
\citet{lai1985asymptotically} proved that any consistent bandit policy $\pol$ must incur at least logarithmic growth ($\Omega(\log T)$) in expected cumulative regret. This lower bound is problem-dependent as it has a multiplicative factor dependent on the given environment $\model$.
\citet{vogel1960asymptotic} proved an environment-independent lower bound of $\Omega(\sqrt{KT})$.
This also called a \emph{problem-independent minimax lower bound} as the minimax regret is  as $\reg_{\mathrm{minimax}}(T) \defn \min_{\pol}\max_{\model}\reg(\pol,\model,T)$. 
A similar lower bound of $\Omega(\sqrt{KT})$ is established for the \emph{Bayesian regret} under any prior \citep{lattimore2018bandit}.

A detailed description of the existing lower bounds is provided in Appendix~\ref{sec:lbounds}.

\section{Differential Privacy in Bandits: Definitions}
\label{sec:dp_mab}
In the classical differential privacy setting, we would like the output distribution of the algorithm to be similar for two neighbouring inputs. However, while it is clear what we should take for the output of the bandit algorithm (the actions), it is unclear whether we should consider the generated outcomes or the observed rewards as the input to the algorithm. In this paper, we will focus on the former definition, because it also implies the latter, and it is also applicable to other settings, like partial monitoring.
\begin{definition}[Differentially Private Bandits] \label{def:DP_mab}
	An algorithm $\pol$ is $\epsilon$-DP if, $\forall t \leq T$:
    \begin{align*}
    &\prob_\pol(a^T \mid \bx_1, \ldots, \bx_t, \ldots, \bx_T) \leq
    \prob_\pol(a^T \mid \bx_1, \ldots, \bx'_t, \ldots, \bx_T) e^\epsilon,
    \end{align*}
	for every sequence of actions $a^T \triangleq [a_1, \ldots, a_T]$ and neighbouring reward sequences $\bx^T$, $\bx'^T$ with $\lp{\bx^T}{\bx'^T}{H}= 1$.
\end{definition}
Let us consider the example of patients contributing their data to a clinical trial or a medical diagnostic experiment. $x_{t,i}$ would correspond to what the physician would see if they were to treat patient $t$ with treatment $i$. Since we cannot know \emph{a priori} which treatment the doctor would choose, we must protect every possible outcome generated by the patients. More concretely, consider actions that are diagnostic tests and where $x_{t,i}$ denotes the absence or presence of a disease marker. Then  potentially revealing individual disease markers clearly violates privacy.

In bandits, we only ever see the $x_{t,i}$ chosen. Thus, we can get away with a weaker privacy definition, as shown in Lemma~\ref{rem:equivalence}. However, because of the conceptual difficulty in referring to the probability of actions given a sequence of rewards that would have only been obtained had those actions been taken, we use the outcomes to define differential privacy. This also allows us to apply this definition to more general sequential decision making problems.
\begin{lemma}
	If and only if a bandit algorithm $\pol$ is $\epsilon$-DP with respect to the outcome sequence $\bx$ then it is $\epsilon$-DP with respect to the reward sequence $r$. However, in the general partial monitoring setting the equivalence does not hold. 
	\label{rem:equivalence}
\end{lemma}

However, a number of authors~\citep{mishra2015nearly,tossou2016algorithms,shariff2018differentially} have suggested differential privacy definitions for bandits of the following form, which we will refer to as \emph{instantaneous} differential privacy.
\begin{definition}[Instantaneous DP Bandits]\label{def:IDP_mab}
	An algorithm $\pol$ is $\epsilon$-Instantaneous DP if, $\forall t \leq T$ and $\{a_t\}, \{r_t\}, \{r_t'\}$,
    \begin{align*}
    &\prob_\pol(a_T \mid a^{T-1}, r_1, \ldots, r_t, \ldots, r_{T-1})
    \leq
    \prob_\pol(a_T \mid a^{T-1}, r_1, \ldots, r'_t, \ldots, r_{T-1})
    e^\epsilon.
    \end{align*}
\end{definition}
Through Remark~\ref{rem:equivalence}, this definition is equivalent to
 \begin{align*}
&\prob_\pol(a_T \mid a_1, \ldots, a_{T-1}, \bx_1, \ldots, \bx_t, \ldots, \bx_{T-1})
\leq
\prob_\pol(a_T \mid a_1, \ldots, a_{T-1}, \bx_1, \ldots, \bx'_t, \ldots, \bx_{T-1})e^\epsilon
\end{align*}
Consequently, it is easy to show that instantaneous differential privacy implies differential privacy and vice-versa. 
\begin{lemma}
	If a policy $\pol$ satisfies $\epsilon$-DP (Definition~\ref{def:DP_mab}) with privacy level $\epsilon$, $\pol$ will also satisfy $\epsilon$-instantaneous DP (Definition~\ref{def:IDP_mab}) with privacy level $2\epsilon$. Conversely, if a policy satisfies $\epsilon$-instantaneous DP, it achieves $t \epsilon$ DP after $t$ steps.
    \label{lem:seq-ins-priv}
\end{lemma}

\section{Mechanisms for Achieving Differential Privacy}\label{sec:dp_mech}
One method to achieve differential privacy is to rely on differentially private inputs to the algorithm. This is called the \emph{Local Differential Privacy} model~\citep[c.f.][]{duchi2013local}. Local DP allows the algorithm to be agnostic about privacy. For this reason, this notion is presently adapted by Apple and Google for their large-scale systems.

If the bandit algorithm only observes a private version of the reward sequence, the algorithm's output is differentially private with respect to the \emph{generated outcome sequence} $\lbrace \bx_1, \ldots, \bx_T \rbrace = \lbrace \bx_i\rbrace_{i=1}^T$. In particular, let the input to the algorithm at time $t$ be $\bz_t$, so that the complete input sequence is $\bz_1, \ldots, \bz_t$. Then if a pre-processing mechanism $\mech$ generates differentially private inputs, i.e.
\begin{align*}
&\prob_\mech(\bz_1, \ldots, \bz_T \mid r_1, \ldots, r_t, \ldots, r_T)
\leq
\prob_\mech(\bz_1, \ldots, \bz_T \mid r_1, \ldots, r_t', \ldots, r_T) e^\epsilon
\end{align*}
then the resulting algorithm $\pol(a_1, \ldots, a_T \mid \bz_1, \ldots \bz_T)$ is differentially private with respect to $\bx$ (and so $r$) through post-processing.

\begin{definition}[$\epsilon$-Local Differentially Private Bandits]\label{def:LDP_mab}
  A bandit algorithm $\pol$ is locally differentially private if its inputs are generated through an $\epsilon$-DP mechanism $\mech$.
\end{definition}
Here, we are not specifying how the input should be related to the rewards. If the input consists of a  private reward sequence $\hat{r}$, the setting of Definition~\ref{def:LDP_mab} is analogous to the local privacy definition in the corrupted bandit setting of~\citet{gajane2017corrupt}. 
If we only maintain differentially private statistics about each arm, i.e. we retain a differentially private mean for the reward of each arm, and use i to select new arms, then we end up with local differential privacy in the way it was used in~\cite{mishra2015nearly}. 

%
\begin{figure*}[t!]
\centering
\scalebox{0.8}{	
\begin{tabular}{cccc}
\begin{tikzpicture}[x=1.7cm,y=1.8cm]
  \node[obs] (a1) {$a_t$}; %
  \node[obs, above=of a1] (x1) {$r_t$} ; %
  \node[latent, above=of x1] (r1) {$\bx_t$} ; %

  \node[obs, left=of a1] (a2) {$a_{t-1}$}; %
  \node[obs, left=of x1] (x2) {$r_{t-1}$} ; %
  \node[latent, left=of r1] (r2) {$\bx_{t-1}$} ; %
  
  \node[latent, above=of r1, xshift=-1.2cm] (t) {$\model$}; %
  \node[latent, below=of a1, xshift=-1.2cm] (p) {$\pol$}; %
  
  \edge[] {t} {r1,r2} ;
  \edge[] {p} {a1,a2} ;
  \edge[] {a2} {a1,x2} ;
  \edge[] {a1} {x1} ;
  \edge[] {x2} {a1} ;
  \edge[] {r1} {x1} ;
  \edge[] {r2} {x2} ;
\end{tikzpicture}&
\begin{tikzpicture}[x=1.7cm,y=1.8cm]
  \node[obs] (a1) {$a_t$}; %
  \node[obs, above=of a1] (x1) {$r_t$} ; %
  \node[latent, above=0.5of x1] (z1) {$\bz_t$} ; %
  \node[latent, above=0.5 of z1] (r1) {$\bx_t$} ; %
  
  \node[obs, left=of a1] (a2) {$a_{t-1}$}; %
  \node[obs, left=of x1] (x2) {$r_{t-1}$} ; %
  \node[latent, left=of z1] (z2) {$\bz_{t-1}$} ; %
  \node[latent, left=of r1] (r2) {$\bx_{t-1}$} ; %
  
  \node[latent, above=of r1, xshift=-1.2cm] (t) {$\model$}; %
  \node[latent, below=of a1, xshift=-1.2cm] (p) {$\pol$}; %
  
  \edge[] {t} {r1,r2} ;
  \edge[] {p} {a1,a2} ;
  \edge[] {a2} {a1,x2} ;
  \edge[] {a1} {x1} ;
  \edge[] {x2} {a1} ;
  \edge[] {r1} {z1} ;
  \edge[] {r2} {z2} ;
  \edge[] {z1} {x1} ;
  \edge[] {z2} {x2} ;
  
  \privplate [inner sep=0.5em] {plate1} {(r1)(r2)} {$\epsilon$-Local Privacy};
\end{tikzpicture}&
\begin{tikzpicture}[x=1.7cm,y=1.8cm]
  \node[obs] (a1) {$a_t$}; %
  \node[obs, above=of a1] (x1) {$r_t$} ; %
  \node[latent, above=of x1] (r1) {$\bx_t$} ; %
  
  \node[obs, left=of a1] (a2) {$a_{t-1}$}; %
  \node[obs, left=of x1] (x2) {$r_{t-1}$} ; %
  \node[latent, left=of r1] (r2) {$\bx_{t-1}$} ; %
  
  \node[latent, above=of r1, xshift=-1.2cm] (t) {$\model$}; %
  \node[latent, below=of a1, xshift=-1.2cm] (p) {$\pol$}; %
  
  \edge[] {t} {r1,r2} ;
  \edge[] {p} {a1,a2} ;
  \edge[] {a2} {a1,x2} ;
  \edge[] {a1} {x1} ;
  \edge[] {x2} {a1} ;
  \edge[] {r1} {x1} ;
  \edge[] {r2} {x2} ;
  
  \privplate [inner sep=0.5em] {plate1} {(r1)(r2)} {$\epsilon$-Differential Privacy};
\end{tikzpicture}&%
\begin{tikzpicture}[x=1.7cm,y=1.8cm]
  \node[obs] (a1) {$a_t$}; %
  \node[obs, above=of a1] (x1) {$r_t$} ; %
  \node[latent, above=of x1] (r1) {$\bx_t$} ; %
  
  \node[obs, left=of a1] (a2) {$a_{t-1}$}; %
  \node[obs, left=of x1] (x2) {$r_{t-1}$} ; %
  \node[latent, left=of r1] (r2) {$\bx_{t-1}$} ; %
  
  \node[latent, above=of r1, xshift=-1.2cm] (t) {$\model$}; %
  \node[latent, below=of a1, xshift=-1.2cm] (p) {$\pol$}; %
  
  \edge[] {t} {r1,r2} ;
  \edge[] {p} {a1,a2} ;
  \edge[] {a2} {a1,x2} ;
  \edge[] {a1} {x1} ;
  \edge[] {x2} {a1} ;
  \edge[] {r1} {x1} ;
  \edge[] {r2} {x2} ;
  
  \privplate [inner sep=.5em] {plate1} {(t)} {$\epsilon$-Environment Privacy};
\end{tikzpicture}\\
(a)Non-privacy&(b)Local Privacy&(c) Differential Privacy&(d) Environment Privacy
\end{tabular}%
}
\caption{Graphical models for non-private, local private, differentially private, and environmentally private multi-armed bandits.}\label{fig:pgm_dp_mab}
\end{figure*}
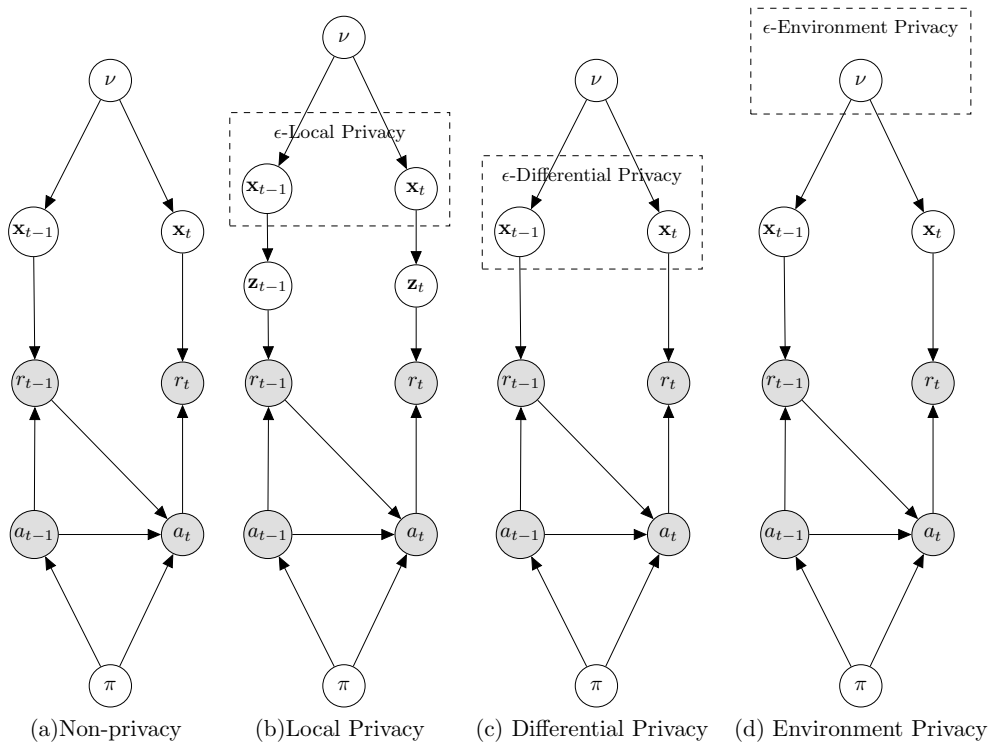
	
\paragraph{Unified Graphical Model Representation of Privacy for Bandits.} Figure~\ref{fig:pgm_dp_mab} provides a unified graphical model perspective of non-private, private (locally and not) multi-armed bandits. The shaded nodes represent the observed variable. The clear nodes represent the hidden variables. The dashed-rectangle covers the input quantities with respect to which the privacy has to be maintained. All of these representations treat the generated rewards as input and all possible action sequences as output with local and  sequential privacy-mechanisms acting at two different levels ensuring local and  sequential privacies respectively. This representation allows us to define a new notion of privacy for bandits, namely \emph{environment privacy}.

\paragraph{Environment Privacy for Bandits.} For a bandit with a stationary environment $\model$, the reward generation mechanism can be represented as a distribution with an environment-dependent parameter $\model$. The user may  consider this \emph{environment parameter to be the input} and \emph{the generated histories} $\gen_T$, i.e. the sequence of generated rewards $\lbrace \br_i\rbrace_{i=1}^T \in \real^{KT}$ and actions taken $\lbrace a_i\rbrace_{i=1}^T \in [K]^T$ by environment parameter $\model$ and the policy $\pol$ as \emph{the output}. 
\begin{definition}[$\epsilon$-Environment Privacy for Bandits]	\label{def:EDP_mab}
	A privacy-preserving mechanism $\mathcal{M}$ preserves $\epsilon$-environment privacy if for all generated histories $\gen_T \defn \lbrace (a_i, \br_i)\rbrace_{i=1}^T \in ([K]\times \real^K)^T$, and environment parameters $\model,\model' \in \real^d$,\vspace{-.5em}
	$$\log\left({\left |\frac{\prob_{\pol\model}\left(\lbrace (a_i, \br_i) \rbrace_{i=1}^T \mid \model\right)}{\prob_{\pol\model'}\left(\lbrace (a_i, \br_i)\rbrace_{i=1}^T \mid \model'\right)}\right|}\right) \leq \epsilon\rho(\model,\model'),$$
	where $\rho$ is a distance metric defined in the space of $\model$.
\end{definition}

\paragraph{An Example of Data Privacy in Bandits.} In order to understand how each definition affects privacy, let us consider the example of web advertising for a specific individual. At time $t$, the individual is presented with some set of advertisements $a_t$. These advertisements generate potential responses $\bx_t$. Out of these generated responses of the user, we only see the clicked response $r_t$. We use a bandit algorithm $\pol$ in order to perform adaptive web advertising. If $\pol$ is $\epsilon$-differentially private (Definition~\ref{def:DP_mab}) with respect to $\{\br_t\}$, an adversary cannot distinguish similar responses between individuals. If we are locally $\epsilon$-DP (Definition~\ref{def:LDP_mab}), even the authority of the algorithm would not be able to distinguish between $\{\br_t\}$. Thus indistinguishability of individuals can be achieved with respect to both the adversary and the authority of the algorithm. 
Finally, if we are $\epsilon$-envrionment private with respect to $\model$ (Definition~\ref{def:EDP_mab}), an adversary cannot infer the inherent preferences of the individuals. 
For all of our definitions, the privacy loss is bounded by a constant privacy level $\epsilon$ independent of the length of interactions.

\section{Regret Lower Bounds for Private Bandits}\label{sec:theory}

\paragraph{Trade-off between Privacy and Performance.} Now, the main question is how much additional regret that we would suffer if we want to satisfy DP. Although a number of DP bandit algorithms with upper bounds on the regret have been proposed, a clear lower bound on the regret that \emph{any} differentially private policy must accrue is lacking. In literature, DP for bandits has been achieved in two ways: (a) through local DP on the (statistics of) rewards, and (b) through instantaneous DP. The first method has the drawback of potentially adding too much noise to the reward signal, while the second method requires that the instantaneous differential privacy level is carefully controlled to achieve a constant (or sublinear) amount of privacy loss.

\paragraph{Regret Lower bound and Intrinsic Hardness of a Bandit Problem.} \textit{Lower bounds on the performance measure of a problem provides us insight about the intrinsic hardness of the problem and sets a target for optimal algorithm design.} In this section, we prove minimax and Bayesian minimax lower bounds for local, instantaneous and general DP bandits. We also prove a problem-dependent lower bound for local DP. Results are listed in Table~\ref{tab:thm}.

\begin{table*}[t!]
	\centering
	\begin{tabular}{c|c|c|c|c}
		\hline 
		Privacy Type & Definition & Minimax  & Bayesian Minimax & Problem-dependent \\
		& for Bandits & Regret &  Regret &  Regret \\  
		\hline 
		DP & Definition~\ref{def:DP_mab} & Theorem~\ref{thm:lb_dp} & Corollary~\ref{cor:lb_bayes_dp} & -- \\ 
		Instantaneous DP & Definition~\ref{def:IDP_mab} & Theorem~\ref{thm:lb_global} & Corollary~\ref{cor:lb_bayes_global} & --  \\ 
		Local DP & Definition~\ref{def:LDP_mab}  & Theorem~\ref{thm:lb_local} & Corollary~\ref{cor:lp_bayes} & Theorem~\ref{thm:lb_log_local} \\ 
		\hline
	\end{tabular} 
	\caption{Regret lower bound results different privacy definitions/mechanisms of bandits.}
\end{table*}

\paragraph{KL-divergence Decomposition in Sequential Setting.} In order to prove the lower bounds, we adopt the general canonical bandit model~\citep{lattimore2018bandit} that is general enough not to impose additional constraints on the bandit algorithm and the environment.
A privacy-preserving bandit algorithm $\pol$ and an environment $\model$ interacts up to a given time horizon $T$ to produce \emph{observed history} $\Hist_T \defn \lbrace (A_i, X_i)\rbrace_{i=1}^T$.
Thus, an observed history $\mathcal{H}_T$ is a random variable sampled from the measurable space $\left(([K]\times \real)^T, \mathcal{B}([K]\times \real)^T\right)$ and a probability measure $\prob_{\pol\model}$. Here, $\mathcal{B}([K]\times \real)^T$ is the Borel set on $([K]\times \real)^T$. $\prob_{\pol\model}$ is the probability measure induced by the algorithm $\pol$ and environment $\model$ such that,
\begin{enumerate}[leftmargin=*]
\item[1.] the probability of choosing an action $A_t = a$ at time $t$ is dictated only by the algorithm $\pol(a|\Hist_{t-1})$,
\item[2.] the distribution of reward $X_t$ is $f_{A_t}$ and is conditionally independent of the previous observed history $\Hist_{t-1}$.
\end{enumerate}
Hence, we get for any observed history $\Hist_T$, 
\begin{align}\label{eqn:indep}
\begin{split}
\prob^T_{\pol{\model}} \defn \prob_{\pol\model}(\Hist_T) = \prod_{t=1}^T \pol(A_t|\Hist_{t-1})f_{A_t}(X_t).
\end{split}
\end{align}
This canonical bandit framework allows us to state Lemma~\ref{thm:kldecomp} on KL-divergence decomposition. This decomposition allows us to separately treat the effect of policy and reward generation. It further leads to Lemma~\ref{thm:kldecomp_local},~\ref{thm:kldecomp_global} and~\ref{thm:kldecomp_DP} for local DP, instantaneous DP, and DP respectively.
\begin{lemma}[KL-divergence Decomposition]\label{thm:kldecomp}
	Given a bandit algorithm $\pol$, two environments $\model_1$ and $\model_2$, and a probability measure $\prob_{\pol\model}$ satisfying Equation~\ref{eqn:indep},\vspace{-.5em}
	\begin{align*}
	\kldiv{\prob^T_{\pol{\model_1}}}{\prob^T_{\pol{\model_2}}} &= \sum_{t=1}^{T} \kldiv{\pol(A_t|\mathcal{H}_t, \model_1)}{\pol(A_t|\mathcal{H}_t, \model_2)}\\ &+ \sum_{a=1}^K \mathbb{E}_{\model_1}\left[N_a(T)\right]\kldiv{f_a \in \model_1}{f_a \in \model_2}.
	\end{align*}
\end{lemma}
For non-private and locally-private algorithms, the first term vanishes and the rest remains. The corresponding equality for non-private bandit algorithms was first proposed in~\citep{garivier2018explore}. The non-private decomposition of~\citep{garivier2018explore} is also used in~\citep{gajane2017corrupt} to derive a regret lower bound for locally private corrupt bandits.
In this paper, we derive two novel upper bounds of the KL-divergence for local DP and DP bandits.
Bounding KL-divergence is a classic technique in minimax risk analysis methods like Le Cam, Fano and Assaoud~\citep{Yu1997}.
Here, we propose three sequential and private versions of that for local, instantaneous, and general DP (Lemma~\ref{thm:kldecomp_local},~\ref{thm:kldecomp_global}, and~\ref{thm:kldecomp_DP}).

\begin{lemma}[Local Private KL-divergence Decomposition]\label{thm:kldecomp_local}
	If the reward generation process is $\epsilon$-local differentially private for both the environments $\model_1$ and $\model_2$,
	\begin{align*}
	\kldiv{\prob^T_{{\model_1}\pol}}{\prob^T_{{\model_2}\pol}} \leq~&2\min\lbrace4,e^{2\epsilon}\rbrace (e^{\epsilon}-1)^2 \sum_{a=1}^K \mathbb{E}_{\model_1}\left[N_a(T)\right]\kldiv{f_a \in \model_1}{f_a \in \model_2}.
	\end{align*}
\end{lemma}
\begin{proof}[Proof Sketch]
	The first term of Lemma~\ref{thm:kldecomp} vanishes for local privacy. The second term yields this upper bound by applying Equation 9 in~\citep{duchi2013local} and Pinsker's inequality consecutively.
\end{proof}

\begin{lemma}[Instantaneously Private KL-divergence Decomposition]\label{thm:kldecomp_global}
	For a sequentially private bandit algorithm $\pol$ satisfying $l(T) \leq \expect_{\pol\model}[N_a(T)]$ for any arm $a$,
	and two environments $\model_1$ and $\model_2$,
	\begin{align*}
	\begin{split}
	\kldiv{\prob^T_{\pol{\model_1}}}{\prob^T_{\pol{\model_2}}} &\leq 2\epsilon(e^{2\epsilon}-1)\frac{1-2e^{-\frac{T}{l(T)}}}{1 - e^{-\frac{T}{l(T)}}} 1+ \sum_{a=1}^K \mathbb{E}_{\model_1}\left[N_a(T)\right](\kldiv{f_a \in \model_1}{f_a \in \model_2}).
	\end{split}
	\end{align*}
\end{lemma}

\begin{lemma}[Differentially Private KL-divergence Decomposition]\label{thm:kldecomp_DP}
	If the policy is $\epsilon$-local differentially private for both the environments $\model_1$ and $\model_2$, we get for two neighbouring histories $\Hist_T$ and $\Hist'_T$ and a constant $c>0$
	\vspace{-.5em}
	\begin{align*}
	\kldiv{\prob_{\pol{\model_1}}(\Hist_T)}{\prob_{\pol{\model_2}}(\Hist_T)} &\leq  2(\epsilon + c) 	+ e^{2(\epsilon + c)} \kldiv{\prob_{\pol{\model_1}}(\Hist'_T)}{\prob_{\pol{\model_2}}(\Hist'_T)} ).
	\end{align*}
\end{lemma}
\begin{proof}[Proof Sketch]
We use Lemma~\ref{lemma:hist_dp} to show that the ratio of two histories would be bounded by $e^{\epsilon + c}$ for a given environment. Here, $c= L \Delta$ such that the reward distributions in the environment $\model$ satisfy: $\ln \sup_{a, x_a,x'_a} \frac{\prob_{\model}(x_a)}{\prob_{\model}(x'_a)} \leq L \Delta$. This allows us to proof this lemma.
\end{proof}

\subsection{Minimax Regret}
Now, we leverage these upper bounds on the KL-divergences to derive the minimax regret bounds for these two cases.
We know from a probabilistic Pinsker's inequality~\citep[Lemma 2.1]{bretagnolle1979estimation} that the probability of an event $E$ and its complement $E^C$ for two distributions $P$ and $Q$ defined on same event space is lower bounded by $\exp(-\kldiv{P}{Q})$.
Thus, the KL-divergence upper bound allows us to lower bound the sum of regret in two environments such that the best policy in one environment is the worst in another.
Now, as we derive this lower bound on sum of regrets in these two environments. We focus on specifying the environments. 
In the end, we choose such suboptimality gaps among rewards such that we can get a lower bound on the best regret in the worst environment.

\begin{theorem}[Local Private Minimax Regret Bound]\label{thm:lb_local}
	Given an $\epsilon$-locally private reward generation mechanism with $\epsilon \in \real$, and a time horizon $T \geq g(K,\epsilon)$, then for any environment with finite variance, any algorithm $\pol$ satisfies
	\begin{equation}
	\reg_{\mathrm{minimax}}(T) \geq \frac{c}{\min\lbrace 2,e^{\epsilon}\rbrace(e^{\epsilon}-1)}\sqrt{(K-1)T}.
	\end{equation}
\end{theorem}

For small $\epsilon$, $e^{\epsilon}-1\approx \epsilon$. Thus, for small $\epsilon$, the minimax regret bound for local privacy worsens by a multiplicative factor $\frac{1}{\epsilon{e^{\epsilon}}}$. If the $\epsilon = 0$ which means the rewards obtained are completely randomised, the arms would not be separable any more and would lead to unbounded minimax regret. For bandits, the regret gets capped by $\Theta(T)$ if the expected rewards of the arms are finite.

\begin{theorem}[Instantaneously Private Minimax Regret Bound]\label{thm:lb_global}
	Given a finite privacy level $\epsilon \leq a/2$, and a time horizon $T \geq h(K,\epsilon)$, then for any environment with finite variance, any algorithm $\pol$ that is $\epsilon$-instantaneous DP satisfies
	\begin{equation}
	\reg_{\mathrm{minimax}}(T) \geq c(a)
	\sqrt{\frac{(K-1)T}{2\epsilon(e^{2\epsilon}-1)}}.
	\end{equation}
\end{theorem}

This implies that for small $\epsilon$, the minimax regret bound for instantaneous DP worsens by a multiplicative factor $\frac{1}{\epsilon}$. Thus, the lower bound of minimax regret for instantaneously private bandit is better than the locally private bandit by factors $e^{\epsilon/2}$ and $\sqrt{\frac{2}{\epsilon(e^{\epsilon} +1)}}$ for small and large $\epsilon$'s respectively.

\begin{theorem}[DP Minimax Regret Bound]\label{thm:lb_dp}
	Given a finite privacy level $\epsilon > 0$, and a time horizon $T \geq h(K,\epsilon)$, then for any environment with finite variance, any algorithm $\pol$ that is $\epsilon$-DP satisfies
	\begin{align*}
	\reg_{\mathrm{minimax}}(T) &\geq \frac{1}{8 e^{3(\epsilon+c)}} \sqrt{\frac{(K-1)T \ln(\epsilon+1) }{\epsilon^{(1+\frac{1}{\epsilon})}(\epsilon^2+1)^{\frac{1}{\epsilon}}}} 
	\end{align*}
	Here, $c=L\Delta$ where $L$ is the Lipschitz constant corresponding to the reward distributions of the environments and $\Delta$ is the difference in generated rewards.
\end{theorem}

The minimax regret lower bound  degrade by a multiplicative factor $\sqrt{\dfrac{\ln(\epsilon^2+1) }{e^{6\epsilon}\epsilon^{(1+\frac{1}{\epsilon})}(\epsilon+B)^{\frac{1}{\epsilon}}}}$ for an $\epsilon$-DP bandit algorithm.
The limit of the lower bound goes to infinity as $\epsilon \rightarrow 0$ i.e. the policy becomes completely random. In such a situation, the regret for bandits gets capped by the order of $T$. As $\epsilon$ increases, the privacy dependent factor in lower bound also decreases monotonically.

\subsection{Bayesian Minimax Regret}
In the Bayesian setup, the bandit algorithm assumes a prior distribution $Q_0$ over the possible environments $\nu \in \mathcal{E}$. As the algorithm $\pi$ plays further and observe corresponding rewards at each time $t$, it updates the prior over the possible environments to a posterior distribution $Q_t$. This framework is adopted for efficient algorithms like Thompson sampling~\citep{thompson1933} and Gittins indices~\citep{gittins1989multi}. In the Bayesian setting, we define the Bayesian regret as $$\reg_{\mathrm{Bayes}}(\pol, T,Q) \defn \int_{\model^T} \reg(\pol, \model, T)dQ(\model).$$
In the Bayesian setting, we define the Bayesian minimax regret as the worst possible regret for any prior
$$\reg^*_{\mathrm{Bayes}}(T) \defn \min_{\pol} \max_Q \int_{\model^T} \reg(\pol, \model, T)dQ(\model)=\min_{\pol}\max_{Q}\reg_{\mathrm{Bayes}}(\pol, T,Q).$$
We use a recent result on equivalence of minimax and Bayesian minimax regrets for bounded rewards to proceed~\cite{lattimore2019information}.

\begin{corollary}[Local Private Bayesian Minimax Regret Bound]\label{cor:lp_bayes}
 Given an $\epsilon$-locally private reward generation mechanism with $\epsilon \in \real$, and a finite time horizon  $T \geq g(K,\epsilon)$, then for any environment with bounded rewards $\br \in [0,1]^K$, any algorithm $\pol$ satisfies\vspace{-.5em}
	\begin{equation}
	\reg^*_{\mathrm{Bayes}}(T) \geq \frac{c}{\min\lbrace 2,e^{\epsilon}\rbrace(e^{\epsilon}-1)}\sqrt{(K-1)T}.
	\end{equation}
\end{corollary}

\begin{corollary}[Instantaneous DP Bayesian Minimax Regret Bound]\label{cor:lb_bayes_global}
	Given a finite privacy level $\epsilon \in \real$, and a finite time horizon  $T \geq h(K,\epsilon)$, then for any environment with finite variance and bounded reward $\br \in [0,1]^K$, any algorithm $\pol$ that is $\epsilon$-instantaneous DP satisfies\vspace*{-.5em}
	\begin{equation}
	\reg^*_{\mathrm{Bayes}}(T) \geq c(\epsilon)
	\sqrt{\frac{(K-1)T}{2\epsilon(e^{2\epsilon}-1)}}.
	\end{equation}
\end{corollary}

\begin{corollary}[DP Bayesian Minimax Regret Bound]\label{cor:lb_bayes_dp}
	Given a finite privacy level $\epsilon \in \real$, and a finite time horizon $T \geq h(K, \epsilon)$, then for any environment with bounded reward $\br \in [0,1]^K$, any algorithm $\pol$ that is $\epsilon$-IDP satisfies
	\begin{align*}
	\reg_{\mathrm{minimax}}(T) &\geq \frac{1}{8 e^{3(\epsilon+c)}} \sqrt{\frac{(K-1)T \ln(\epsilon+1) }{\epsilon^{(1+\frac{1}{\epsilon})}(\epsilon^2+1)^{\frac{1}{\epsilon}}}} 
	\end{align*}
	Here, $c=L\Delta$ where $L$ is the Lipschitz constant corresponding to the reward distributions of the environments and $\Delta$ is the difference in generated rewards.
\end{corollary}

\subsection{Problem-dependent Regret}
Both minimax and Bayesian minimax regret bounds are problem independent. They represent the worst-case regret for any environment and any prior over environments respectively. Someone may want to design algorithms that is optimal for a given environment $\nu$ and the minimax and Bayesian minimax bounds are too pessimistic for them. Thus, researchers study the problem-dependent lower bounds of regret involving environment dependent quantities. \cite{lai1985asymptotically} proved that a bandit algorithm achieves $\Omega(\log T)$ problem-dependent lower bound. 

We prove that for $\epsilon$-local privacy this lower bound worsens by a multiplicative factor $\frac{1}{e^{2\epsilon}(e^{\epsilon}-1)^2}$. It is an open problem for instantaneous and general DP.
\begin{theorem}[Problem-dependent Local-Private Regret Bound]\label{thm:lb_log_local}
 For any asymptotically consistent bandit algorithm $\pol$, an environment $\model$ with optimal reward distribution $f^*$, and an $\epsilon$-locally private reward generation mechanism, the expected cumulative regret\vspace*{-.5em}
 \begin{equation*}\hspace{-1em}
     \liminf_{T \rightarrow \infty} \frac{\reg(\pol, \model, T)}{\log T} \geq \sum_{a \neq a^*} \frac{\Delta_a}{2\min\lbrace4,e^{2\epsilon}\rbrace(e^{\epsilon}-1)^2 \kldiv{f_a}{f^*}}.
 \end{equation*}\vspace*{-1em}
\end{theorem}

%

\subsection{Discussion} 
This shows that for both local and non-local DP in bandits the regret lower bound changes by a multiplicative factor dependent on the privacy level $\epsilon$. The factors are such that the regret goes to infinity if $\epsilon \rightarrow 0$ i.e. as the privacy mechanism gets completely random. In that case, the regret get capped by $T$. The bounds also show that the lower bound for local privacy is worse than that the general case, as expected. It also shows instantaneous DP is worse than the general DP if we maintain a constant $\epsilon$ at all steps $t$. This shows approaching the optimal policy from randomised rewards using the local privacy model is slower than to using a policy that is inherently DP. In Section~\ref{sec:lbounds}, we show our lower bounds falsify the conjecture of an additive factor in the lower bound by~\cite{tossou2016algorithms} and point out existing gaps in optimal algorithm design for private bandits.

\vspace{-1em}
\section{Existing Lower Bounds for Non-private and Private Bandits}\label{sec:lbounds}

\textbf{Problem-independent Non-private Lower Bounds:} Minimax regret is the worst case regret that a bandit algorithm can incur if the environment is unknown. Thus, it is often referred as the problem-independent regret.
\cite{vogel1960asymptotic} performed the first minimax analysis of two-armed Bernoulli bandits. \cite{auer2002nonstochastic} generalised it to $K$-arm Bernoulli distributions. 
\cite{gerchinovitz2016refined} provided a novel technique to establish high probability regret lower bounds for adversarial bandits with bounded reward.
For any bandit algorithm, the minimax regret is lower bounded by 
$\mathrm{Reg}_{\mathrm{minimax}}(T) \geq c \sqrt{(K-1)T}$.
A bandit algorithm $\pol$ is called \emph{minimax optimal} if its minimax regret is upper bounded by $C\sqrt{(K-1)T}$.
In the Bayesian setup, the bandit algorithm assumes a prior distribution $Q_0$ over the possible environments $\nu \in \mathcal{E}$. \cite{lattimore2018bandit} proved that for any bandit algorithm $\pi$ there exists a prior distribution $Q_0$ that the Bayesian regret $\mathrm{Reg}_{\mathrm{Bayes}}(T) \geq C \sqrt{KT}$. This indicates that the minimax regret and the Bayesian regret lower bounds are identical for non-private bandits. This also holds for private bandits. \cite{lattimore2019information} provides the reasoning behind this connection using a modified minimax theorem.

\textbf{Existing Lower Bounds for Differentially Private Bandits.} 
\cite{mishra2015nearly} proposed differentially private variants of UCB and Thompson sampling algorithms. \cite{tossou2016algorithms} improved the differentially private variant of UCB algorithm to obtain regret upper bound of $\Omega(\sum_{a\neq a^*}\frac{\Delta_a}{\kldiv{f_a}{f^*}}\log T + \frac{1}{\epsilon})$ for instantaneous privacy with time varying privacy loss. Our results show that the stricter sequential privacy definition that a constant privacy loss can be achieved with only an additive term on the regret cannot be true.
\cite{shariff2018differentially} proves a finite-time problem-dependent lower bound for contextual bandits. It indicates that the finite-time lower bound of regret is $\Omega(\log T)$ like the non-private  lower bounds but with a modified multiplicative factor $(\sum_{a\neq a^*}\frac{\Delta_a}{\kldiv{f_a}{f^*}}+ \frac{K}{\epsilon})$. Though their definition of privacy is a bit ambiguous as it can be reduced to either of Definition~\ref{def:DP_mab} and~\ref{def:IDP_mab}. Among them, the first being unsuitable for continual observation defeats the purpose for bandits. \emph{We try to clarify at this point with the pan, instantaneous, and sequential privacy definitions.}

\cite{gajane2017corrupt} uses an analogous local privacy definition. They proved a finite-time problem-dependent lower bound of regret for locally private multi-armed bandits where the local privacy is induced by the corrupt bandit mechanism. In this case also, the $\Omega(\log T)$ regret bound of non-private bandits is maintaines with a modified multiplicative factor $\sum_{a\neq a^*}\frac{\Delta_a}{\kldiv{g_a}{g^*}}$. Here, $g_a$ and $g^*$ are the corrupt versions of the reward distributions $f_a$ and $f^*$ ensuring $\epsilon$-local privacy. Theorem~\ref{thm:lb_log_local} shows that the algorithms they have proposed, namely kl-UCB-CF and TS-CF, are suboptimal by at least by a factor $(1+e^{-\epsilon})^2$. \emph{This opens up the problem of designing optimal local-private bandit algorithms.}

We are not aware of any problem-independent minimax lower bounds for (locally and standard) differentially private bandits, before this paper. 
\cite{tossou2017achieving} proposed a differentially private variant of EXP3 algorithm that achieves privacy for adversarial bandits with regret upper bound $O(\frac{\sqrt{T}\log T}{\epsilon})$. \emph{The lower bound of Theorem~\ref{thm:lb_global} shows that designing an optimal private algorithm for adversarial bandits is an open problem.}

\section{Discussion and Future Work}\label{sec:conc}
We provide a unifying set of definitions for differentially private bandit algorithms that distinguishes between local and instantaneous DP used in the literature, as well as for general DP.
We proved corresponding minimax and Bayesian minimax lower bounds for both local and general DP bandit algorithms.
These bounds also pose design of optimal local and sequential private bandit algorithms as open problems since the existing algorithms are suboptimal.
We are now working on deriving the problem-dependent lower bound for DP.
This is based on a generalised KL-divergence decomposition lemma adapted for local and standard differential privacy definitions. Though the literature consists of problem-dependent regret bounds, these are the first minimax and Bayesian regret bounds for both differentially private bandits. We show that both in general and when differential privacy is achieved using a local mechanism, the regret scales as a multiplicative factor of $\epsilon$.

In future, researchers can utilise these lower bounds to design optimal private bandit algorithms for real-life applications, such as recommender systems and clinical trials. We are now working towards the two remaining problem dependent bounds for instantaneous and general DP to complete this problem.

\section*{Acknowledgement}
This work was partially supported by the Wallenberg AI, Autonomous Systems and Software Program (WASP) funded by the Knut and Alice Wallenberg Foundation.

\bibliographystyle{abbrvnat}
\bibliography{reference}

\newpage
\appendix
\section{Proofs of Section~\ref{sec:dp_mab} (Differential Privacy in Bandits: Definitions)}\label{sec:proofs_def}
\begin{lemmarep}
	Iff a bandit algorithm $\pol$ is DP with respect to the outcome sequence $\bx$ then it is DP with respect to the reward sequence $r$. However, in the general partial monitoring setting the equivalence does not hold. 
	\label{rem:equivalence}
\end{lemmarep}
\begin{proof}
	The first claim is direct:
	\begin{align*}
	\pol(a_1, \ldots, a_T \mid \bx_1, \ldots, \bx_T)
	&= \prod_{t=1}^T \pol(a_t \mid a_1, \ldots, a_{t-1}, \bx_1, \ldots, \bx_{t-1})\\
	&= \prod_{t=1}^T \pol(a_t \mid a_1, \ldots, a_{t-1}, x_{1,a_1}, \ldots, x_{t-1,a_{t-1}})\\
	&= \prod_{t=1}^T \pol(a_t \mid a_1, \ldots, a_{t-1}, r_1, \ldots, r_{t-1})\\
	&=  \pol(a_1, \ldots, a_T \mid r_1, \ldots, r_T)
	\end{align*}
	The second follows from the fact that in the general setting the algorithm may observe the complete outcome $\bx_t$ even if the reward it obtains only depends on part of $r_t$.
\end{proof}

\begin{lemmarep}
	If a policy $\pol$ satisfies differential privacy (Definition~\ref{def:DP_mab}) with privacy level $\epsilon$, $\pol$ will also satisfy instantaneous privacy (Definition~\ref{def:IDP_mab}) with privacy level $2\epsilon$. Conversely, if a policy satisfies $\epsilon$ instantaneous privacy, it only achieves $t \epsilon$ differential privacy after $t$ steps.
\end{lemmarep}
\begin{proof}
	For simplicity, we write $a^t = \{a_1, \ldots, a_t\}$ and $\br^t = \{\br_1, \ldots, \br_t\}$ for reward and action sequences respectively.
	
	If $\pol$ is $2 \epsilon$-instantaneous private then (by definition) the following ratio must be bounded from above by $e^{2 \epsilon}$ for any two neighbouring reward sequences $\br^{t}, \hat{\br}^{t}$ :
	\begin{align*}
	\frac{\pol(a_t \mid a^{t-1}, \br^{t-1})}{\pol(a_t \mid a^{t-1}, \hat{\br}^{t-1})}
	&\underset{(a)}{=}
	\frac{\pol(a^t \mid \br^{t-1})}{\pol(a^t \mid \hat{\br}^{t-1})}
	\frac{\pol(a^{t-1} \mid \hat{\br}^{t-1})}{\pol(a^{t-1} \mid \br^{t-1})}\\
	&\underset{(b)}{=}
	\frac{\pol(a^t \mid \br^{t-1})}{\pol(a^t \mid \hat{\br}^{t-1})}
	\frac{\pol(a^{t-1} \mid \hat{\br}^{t-2})}{\pol(a^{t-1} \mid \br^{t-2})}\\
	&\underset{(c)}{\leq}
	e^{2 \epsilon}.
	\end{align*}
	The equality in step (a) is obtained through the definition of conditional probability.
	Step (b) is resulted by independence of actions on current rewards.
	The final inequality is through assumption of $\epsilon$-differential privacy and that $\br, \hat{\br}$ are neighbours.
	The converse follows from composition of differential privacy~\citep[Corollary 3.15]{dwork2014algorithmic}.
\end{proof}

\section{Proofs of Section~\ref{sec:theory} (Regret Lower Bounds for Private Bandits)}\label{sec:proofs_lb}
First, let us remind ourselves of the chain rule of KL divergence for two probability measures $P,Q$ on a product space $\CX^T$ for a given $T$:
\begin{align*}
\kldiv{P}{Q}
&\defn
\int\limits_{\CX^T} \ln \frac{\dd P(x^T)}{\dd Q(x^T)} \dd P(x^T)\\
&= 
\int\limits_{\CX^T} \ln \frac{\dd [P(x_T \mid x^{T-1}) P(x^{T-1})]}{\dd [Q(x_T \mid x^{x-1}) Q(x^{T-1})]} \dd [P(x_T \mid x^{T-1}) P(x^{T-1})]
\\
&= 
\int\limits_{\CX^T} \left[\ln \frac{\dd P(x_T \mid x^{T-1})}{\dd Q(x_T \mid x^{x-1})} + \ln \frac{\dd P(x^{T-1})}{\dd Q(x^{T-1})}\right] \dd [P(x_T \mid x^{T-1})  P(x^{T-1})]
\\
&= 
\int\limits_{\CX^T} \ln \frac{\dd P(x_T \mid x^{T-1})}{\dd Q(x_T \mid x^{x-1})} \dd P(x^{T}) + \int\limits_{\CX^{T-1}} \ln \frac{\dd P(x^{T-1})}{\dd Q(x^{T-1})} \dd P(x^{T-1})
\\
&= 
\sum_{t=1}^T \int\limits_{\CX^t} \ln \frac{\dd P(x_t \mid x^{t-1})}{\dd Q(x_t \mid x^{t-1})} \dd P(x^{t}) \\
&= 
\sum_{t=1}^T 
\expect_{P(x^{t-1})}\left[\kldiv{P(x_t \mid x^{t-1})}{Q(x_t \mid x^{t-1})}\right]. 
\end{align*}
Here, $x_t$ denotes the reward at time $t$ and $x^{t}$ denotes the sequence of rewards obtained from the beginning to time $t$, i.e. $\lbrace x_1, \ldots, x_t\rbrace$. Now, the conditional KL-divergence is defined here as 
\[
\kldiv{P(x\mid y)}{Q(x\mid y)} \defn \int_{\CX \times \CY} \ln \frac{\dd P(x\mid y)}{\dd Q(x\mid y)} \dd P(x, y).
\]
Thus, we get the chain rule of KL-divergence
\begin{align*}
\begin{split}
\kldiv{P}{Q}
&\defn
\int\limits_{x^T \in \CX^T} \ln \frac{P(x^T)}{Q(x^T)} \dd P(x^T)\\
&= 
\sum_{t=1}^T \kldiv{P(x_t \mid x^{t-1})}{Q(x_t \mid x^{t-1})}. 
\end{split}
\end{align*} 
\begin{lemmarep}[KL-divergence Decomposition]
	Given a bandit algorithm $\pol$, two environments $\model_1$ and $\model_2$, and a probability measure $\prob_{\pol\model}$ satisfying Equation~\ref{eqn:indep},
	\begin{align}\label{eqn:kldecomp}
	\begin{split}
	\kldiv{\prob^T_{\pol{\model_1}}}{\prob^T_{\pol{\model_2}}} &= \sum_{t=1}^T \mathbb{E}_{\pol\model_1}\left[\kldiv{\pol(A_t|\mathcal{H}_t, \model_1)}{\pol(A_t|\mathcal{H}_t, \model_2)}\right] + \sum_{a=1}^K \mathbb{E}_{\pol\model_1}\left[N_a(T)\right]\kldiv{f_a \in \model_1}{f_a \in \model_2}.
	\end{split}
	\end{align}
\end{lemmarep}
\begin{proof}
	\begin{align*}
	&\kldiv{\prob^T_{\pol{\model_1}}}{\prob^T_{\pol{\model_2}}}\\
	\underset{(a)}{=} &\sum_{t=1}^{T} \expect_{\prob^T_{\pol{\model_1}}} [ \kldiv{\pol(A_t|\mathcal{H}_t, \model_1)}{\pol(A_t|\mathcal{H}_t, \model_2)} + \kldiv{f(X_t|A_t,\mathcal{H}_t, \model_1)}{f(X_t|A_t,\mathcal{H}_t, \model_2)} ]\\
	\underset{(b)}{=} &\sum_{t=1}^T \expect_{\prob^T_{\pol{\model_1}}}\left[\kldiv{\pol(A_t|\mathcal{H}_t, \model_1)}{\pol(A_t|\mathcal{H}_t, \model_2)}\right] + \sum_{t=1}^{T} \expect_{\prob^T_{\pol{\model_1}}}\left[ \sum_{a=1}^K \mathds{1}_{A_t=a} \kldiv{f_a(X_t) \in \model_1}{f_a(X_t) \in \model_2}\right]\\
	\underset{(c)}{=} &\sum_{t=1}^T \expect_{\prob^T_{\pol{\model_1}}}\left[\kldiv{\pol(A_t|\mathcal{H}_t, \model_1)}{\pol(A_t|\mathcal{H}_t, \model_2)}\right] + \sum_{a=1}^K \left[\sum_{t=1}^{T} \expect_{\prob^T_{\pol{\model_1}}}[\mathds{1}_{A_t=a}] \kldiv{f_a(X_t) \in \model_1}{f_a(X_t) \in \model_2}\right]\\
	\underset{(d)}{=} &\sum_{t=1}^T \expect_{\prob^T_{\pol{\model_1}}}\left[\kldiv{\pol(A_t|\mathcal{H}_t, \model_1)}{\pol(A_t|\mathcal{H}_t, \model_2)}\right] + \sum_{a=1}^K \expect_{\prob^T_{\pol{\model_1}}}\left[N_a(T)\right]\kldiv{f_a \in \model_1}{f_a \in \model_2}.
	\end{align*}
	The equality in step (a) is followed by the chain rule of KL-divergence and Equation~\ref{eqn:indep}.
	Step (b) is obtained by conditioning.
	Step (c) is a consequence of the linearity of expectation.
	The equality in step (d) comes from the fact that expectation of an indicator function of an event returns its probability of occurrence.
\end{proof}
This style of KL-divergence decomposition appeared in proofs of~\citep{auer2002nonstochastic,garivier2018explore,lattimore2018bandit}. We adopt the proof in our context and notations with enough generality to proof the differentially private versions of it later.

\ifdefined \workshop
\subsection{Lower Bounds for Local Differential Privacy (LDP)}
\else
\subsection{Proofs for Local Differential Privacy (LDP)}
\fi
\ifdefined \workshop
\paragraph{Sketch of Results and Discussion.} 
At first, we derive an upper bound on the KL-divergence decomposition for locally private reward generation. This upper bound would allow us to lower bound the sum of regret in two environments such that the best policy in one is the worst in another.
\begin{lemmarep}[Local Private KL-divergence Decomposition]
	If the reward generation process is $\epsilon$-local differentially private for both the environments $\model_1$ and $\model_2$,\vspace{-.5em}
	\begin{align*}
	\kldiv{\prob^T_{{\model_1}\pol}}{\prob^T_{{\model_2}\pol}} &\leq~2\min\lbrace4,e^{2\epsilon}\rbrace (e^{\epsilon}-1)^2\\ &\sum_{a=1}^K \mathbb{E}_{\pol \model_1}\left[N_a(T)\right]\kldiv{f_a \in \model_1}{f_a \in \model_2}.
	\end{align*}
\end{lemmarep}
Now, as we derive this lower bound on sum of regrets in these two environments. We focus on specifying the environments. 
Specifically, we choose such suboptimality gaps among rewards such that we can get a lower bound on the best regret in the worst environment.
\begin{theoremrep}[Local Private Minimax Regret Bound]
	Given an $\epsilon$-locally private reward generation mechanism with $\epsilon \in \real$, and a time horizon $T \geq g(K,\epsilon)$, then for any environment with finite variance, any algorithm $\pol$ satisfies\vspace{-1em}
	\begin{equation}
	\reg_{\mathrm{minimax}}(T) \geq \frac{c}{\min\lbrace 2,e^{\epsilon}\rbrace(e^{\epsilon}-1)}\sqrt{(K-1)T}.
	\end{equation}\vspace*{-1.5em}
	\end{theoremrep}
	For small $\epsilon$, $e^{\epsilon}-1\approx \epsilon$. Thus, for small $\epsilon$, the minimax regret bound for local privacy worsens by a multiplicative factor $\frac{1}{\epsilon{e^{\epsilon}}}$. If the $\epsilon = 0$ which means the rewards obtained are completely randomised, the arms would not be separable any more and would lead to unbounded minimax regret.
	
	Following this, we establish a lower bound for Bayesian regret of local private bandits.
	In the Bayesian setup, the bandit algorithm assumes a prior distribution $Q_0$ over the possible environments $\nu \in \mathcal{E}$. As the algorithm $\pi$ plays further and observe corresponding rewards at each time $t$, it updates the prior over the possible environments to a posterior distribution $Q_t$. This framework is adopted for efficient algorithms like Thompson sampling~\citep{thompson1933} and Gittins indices~\citep{gittins1989multi}. In the Bayesian setting, we define the Bayesian regret as $\reg_{\mathrm{Bayes}}(\pol, T,Q) \defn \int_{\model^T} \reg(\pol, \model, T)dQ(\model)$.
	\begin{corollaryrep}[Local Private Bayesian Minimax Regret Bound]
	Given an $\epsilon$-locally private reward generation mechanism with $\epsilon \in \real$, and a finite time horizon  $T \geq g(K,\epsilon)$, then for any environment with bounded rewards $\br \in [0,1]^K$, any algorithm $\pol$ satisfies\vspace{-.5em}
		\begin{equation}
			\reg^*_{\mathrm{Bayes}}(T) \geq \frac{c}{\min\lbrace 2,e^{\epsilon}\rbrace(e^{\epsilon}-1)}\sqrt{(K-1)T}.
		\end{equation}\vspace*{-1.5em}
		\end{corollaryrep}
	Both minimax and Bayesian minimax regret bounds are problem independent. They represent the worst-case regret for any environment and any prior over environments respectively. Someone may want to design algorithms that is optimal for a given environment $\nu$ and the minimax and Bayesian minimax bounds are too pessimistic for them. Thus, researchers study the problem-dependent lower bounds of regret involving environment dependent quantities. \cite{lai1985asymptotically} proved that a bandit algorithm achieves $\Omega(\log T)$ problem-dependent lower bound. We prove that for $\epsilon$-local privacy this lower bound worsens by a multiplicative factor $\frac{1}{e^{2\epsilon}(e^{\epsilon}-1)^2}$.
	\begin{theoremrep}[Problem-dependent Local-Private Regret Bound]
	For any asymptotically consistent bandit algorithm $\pol$, an environment $\model$ with optimal reward distribution $f^*$, and an $\epsilon$-locally private reward generation mechanism, the expected cumulative regret\vspace*{-.5em}
				\begin{equation*}\hspace{-1em}
				\liminf_{T \rightarrow \infty} \frac{\reg(\pol, \model, T)}{\log T} \geq \sum_{a \neq a^*} \frac{\Delta_a}{2\min\lbrace4,e^{2\epsilon}\rbrace(e^{\epsilon}-1)^2 \kldiv{f_a}{f^*}}.
				\end{equation*}\vspace*{-1em}
	\end{theoremrep}
				
\paragraph{Proof Details.}		
\else
\
\fi		
\begin{lemmarep}[Locally Private KL-divergence Decomposition]
	If the reward generation process is $\epsilon$-local differentially private for both the environments $\model_1$ and $\model_2$,
	\begin{align*}
	\begin{split}
	&\kldiv{\prob^T_{{\model_1}\pol}}{\prob^T_{{\model_2}\pol}} \leq 2\min\lbrace4,e^{2\epsilon}\rbrace (e^{\epsilon}-1)^2 \sum_{a=1}^K \mathbb{E}_{\pol \model_1}\left[N_a(T)\right]\kldiv{f_a \in \model_1}{f_a \in \model_2}.
	\end{split}
	\end{align*}
\end{lemmarep}
\begin{proof}
	In case of local privacy, the reward observed by the algorithm $\pi$ is obtained at time $t$, $\CX_t$, from a privatised version of generated rewards $\BZ_t$. Thus, $x_t = z_{t,a}$, where $a$ refers to the action selected at time $t$. We denote the distribution over the privatised generated reward of arm $a$ as $g_a(z)$. $g_a(z)$ is obtained by imposing a local privacy mechanism $\mech$ on the original reward distribution $f_a(z)$.
	
	We note that the KL-divergence decomposition of Lemma~\ref{thm:kldecomp} is obtained on the probability space over observed histories. Since the observed rewards are now coming from $g_a(z)$ rather than $f_a(x)$, we begin our derivations of local-private KL-divergence decomposition by substituting $g_a(z)$ in Equation~\ref{eqn:kldecomp}. For brevity, we denote $g_a^1$ and $g_a^2$ to represent the privatised reward distributions for arm $a$ in environments $\model_1$ and $\model_2$ respectively. Thus,
	\begin{align*}
	\kldiv{\prob^T_{\pol{\model_1}}}{\prob^T_{\pol{\model_2}}}
	\underset{(a)}{=} &\sum_{t=1}^T \mathbb{E}_{\pol\model_1}\left[\kldiv{\pol(A_t|\mathcal{H}_t, \model_1)}{\pol(A_t|\mathcal{H}_t, \model_2)}\right] + \sum_{a=1}^K \mathbb{E}_{\pol\model_1}\left[N_a(T)\right]\kldiv{g_a^1(Z)}{g_a^2(Z)}\\
	\underset{(b)}{=} & \sum_{a=1}^K \mathbb{E}_{\pol\model_1}\left[N_a(T)\right]\kldiv{g_a^1(Z)}{g_a^2(Z)}\\
	\underset{(c)}{\leq} & \sum_{a=1}^K \mathbb{E}_{\pol\model_1}\left[N_a(T)\right]\left[\kldiv{g_a^1(Z)}{g_a^2(Z)}+\kldiv{g_a^2(Z)}{g_a^1(Z)}\right]\\
	\underset{(d)}{\leq} &\min\lbrace4,e^{2\epsilon}\rbrace (e^{\epsilon}-1)^2 \sum_{a=1}^K \mathbb{E}_{\pol \model_1}\left[N_a(T)\right]\lp{f_a^1(X)}{f_a^2(X)}{TV}^2\\
	\underset{(e)}{\leq} &2\min\lbrace4,e^{2\epsilon}\rbrace (e^{\epsilon}-1)^2 \sum_{a=1}^K \mathbb{E}_{\pol \model_1}\left[N_a(T)\right]\kldiv{f_a^1(X)}{f_a^2(X)}
	\end{align*}
	Step (a) is from the generic KL-divergence decomposition in Lemma~\ref{thm:kldecomp}.
	Step (b) is due to the fact that given the same history, $\pol(A_t|\mathcal{H}_t, \model_1)=\pol(A_t|\mathcal{H}_t, \model_2)$ as they do not vary with the model and depends only on the internal randomisation of the algorithm $\pi$.
	The inequality in step (c) is derived from non-negativity of KL-divergence~\citep{cover2012elements} and the fact that for two non-negative numbers $a$ and $b$, $a \leq a+b$.
	The inequality in the step (d) is obtained from Theorem 1 in~\citep{duchi2013local}.
	The inequality in step (e) is obtained by applying Pinsker's inequality~\citep{cover2012elements}. Pinsker's inequality states that for any two distributions $P$ and $Q$, square of their total variance distance is upper bounded by $2$ times their KL-divergence: $\lp{P}{Q}{TV}^2 \leq 2\kldiv{P}{Q}$.
\end{proof}
\begin{theoremrep}[Locally Private Minimax Regret Bound]
	Given an $\epsilon$-locally private reward generation mechanism with $\epsilon \in \real$, and a time horizon $T \geq g(K,\epsilon)$, then for any environment with finite variance, any algorithm $\pol$ satisfies
	\begin{equation}
	\reg_{\mathrm{minimax}}(T) \geq \frac{c}{\min\lbrace 2,e^{\epsilon}\rbrace(e^{\epsilon}-1)}\sqrt{(K-1)T}.
	\end{equation}
\end{theoremrep}
\begin{proof}\ \\

	\textbf{Step 1:} Fix two environments $\model_1$ and $\model_2$ such that drawing arm $1$ in $\model_1$ for more than $T/2$ times is good but doing the same is bad for $\model_2$. 
	
	We define $\model_1$ to be a set of $K$-reward distribution with mean reward $\lbrace \Delta, 0, \ldots, 0\rbrace$ and finite Fisher information $I$. Similarly, we define $\model_2$ to be to be a set of $K$-reward distribution with mean reward $\lbrace \Delta, \ldots, 0, 2\Delta\rbrace$ and finite Fisher information $I$. Drawing arm $1$ is the optimal choice in $\model_1$ whereas drawing arm $K$ is the optimal choice in $\model_2$.
	
	\textbf{Step 2:} We get the lower bounds of expected cumulative regret for a policy $\pol$, and the environments $\model_1$ and $\model_2$ as follows:
	\begin{align*}
	\reg(\pol, \model_1, T) &\geq \prob_{\pol\model_1}\left(N_1(T) \leq T/2\right) \frac{T\Delta}{2},\\
	\reg(\pol, \model_2, T) &> \prob_{\pol\model_2}\left(N_1(T) > T/2\right) \frac{T\Delta}{2}.
	\end{align*}
	Let us denote the event $N_1(T) \leq T/2$ as $E$.
	Hence, we get
	\begin{align}
	\reg(\pol, \model_1, T) + \reg(\pol, \model_2, T) &\geq  \frac{T\Delta}{2} \left(\prob_{\pol\model_1}(E) + \prob_{\pol\model_2}(E^C)\right) \notag \\ 
	&\geq \frac{T\Delta}{4} \exp(-\kldiv{\prob_{\pol\model_1}}{\prob_{\pol\model_2}}).\label{eqn:regret_to_KL}
	\end{align}
	We obtain the last inequality from the Lemma 2.1 in~\citep{bretagnolle1979estimation}. This is called Bretagnolle-Huber inequality or probabilistic Pinsker's inequality~\citep{auer2002nonstochastic,lattimore2019information} and used for several proofs of bandit algorithms. This states that for any two distributions $P$ and $Q$ defined on the same measurable space and an event $E$, $P(E)+Q(E^C) \geq \exp(-\kldiv{P}{Q})$.
	
    \textbf{Step 3:} From Lemma~\ref{thm:kldecomp_local}, we get 
	\begin{align}
	\kldiv{\prob_{\pol\model_1}}{\prob_{\pol\model_2}}
	&= \expect_{\pol\model_1}[N_K(T)]\kldiv{\mech(f_K(0,I))}{\mech(f_K(2\Delta,I))} \notag\\
	&\underset{(a)}{\leq} 2\min\lbrace4,e^{2\epsilon}\rbrace (e^{\epsilon}-1)^2 \mathbb{E}_{\pol\model_1}\left[N_K(T)\right]\kldiv{f_K(0,I)}{f_K(2\Delta,I)} \notag\\
	&\leq 2\min\lbrace4,e^{2\epsilon}\rbrace (e^{\epsilon}-1)^2 \mathbb{E}_{\pol\model_1}\left[N_K(T)\right] (2\Delta^2) \notag\\
	&\leq 4\min\lbrace4,e^{2\epsilon}\rbrace (e^{\epsilon}-1)^2 \frac{T\Delta^2}{K-1}.    \label{eqn:kl_local}
	\end{align}
	Step (a) is a consequence of Lemma~\ref{thm:kldecomp_local}.
	
	From Equations~\eqref{eqn:regret_to_KL} and~\eqref{eqn:kl_local}, we obtain  the regret bound to be
	\begin{align*}
	\max\lbrace \reg(\pol, \model_1, T), \reg(\pol, \model_2, T) \rbrace
	&\geq \frac{1}{2}\left(\reg(\pol, \model_1, T) + \reg(\pol, \model_2, T)\right)\\
	&\geq \frac{T\Delta}{4}\exp\left[-4\min\lbrace4,e^{2\epsilon}\rbrace (e^{\epsilon}-1)^2 \frac{T\Delta^2}{K-1}\right].
	\end{align*}
	
	\textbf{Step 4:} Let us choose the optimality gap $\Delta$ to be $\sqrt{\frac{(K-1)}{\min\lbrace 4,e^{2\epsilon}\rbrace (e^{\epsilon}-1)^2 T}} \leq \frac{1}{2}$. This holds for any for $T \geq \frac{(K-1)}{\min\lbrace 4,e^{2\epsilon}\rbrace (e^{\epsilon}-1)^2} \defn g(K,\epsilon)$. Hence, by using the results of Step 3, we obtain:
	\begin{equation*}
	\reg_{\mathrm{minimax}}(T) \geq \frac{1}{4e^{4}} \sqrt{\frac{(K-1)T}{\min\lbrace 4,e^{2\epsilon}\rbrace (e^{\epsilon}-1)^2}} = \dfrac{\sqrt{(K-1)T}}{4e^{4}\min\lbrace 2, e^{\epsilon}\rbrace (e^{\epsilon}-1)}.
	\end{equation*}
	For $0 < c\leq \frac{1}{4e^{4}}$, we obtain the theorem statement.
\end{proof}

\begin{corollaryrep}[Locally Private Bayesian Minimax Regret Bound]
	Given an $\epsilon$-locally private reward generation mechanism with $\epsilon \in \real$, and a finite time horizon  $T \geq g(K,\epsilon)$, then for any environment with bounded rewards $\br \in [0,1]^K$, any algorithm $\pol$ satisfies\vspace{-.5em}
	\begin{equation}
	\reg^*_{\mathrm{Bayes}}(T) \geq \frac{c}{\min\lbrace 2,e^{\epsilon}\rbrace(e^{\epsilon}-1)}\sqrt{(K-1)T}.
	\end{equation}
\end{corollaryrep}
\begin{proof}
	Let us denote the space of all plausible priors for a given bandit problem to be $\mathcal{L} \defn \lbrace \mu \rbrace$.
	\cite{lattimore2019information} proves in their recent paper that if $\mathcal{L}$ is the space of all finitely supported probability measures on $[0,1]^{KT}$, the Bayesian regret and the minimax regret would be the same.
	
	\begin{fact}[Theorem~1 in~\citep{lattimore2019information}]\label{thm:lattimore}
		Let $\mathcal{L}$ be the space of all finitely supported probability measures on $\mathcal{R}^T$, where $\mathcal{R} \defn [0,1]^K$. Then
		\begin{equation*}
		\reg_{\mathrm{minimax}}(T) = \reg^*_{\mathrm{Bayes}}(T).
		\end{equation*}
	\end{fact}
	Since variance of bounded random variable in $[0,1]$ is less than $\frac{1}{4}$, the bounded reward assumption satisfies the finite variance requirement of Theorem~\ref{thm:lb_local}. 
	
	Thus, the results of Theorem~\ref{thm:lb_local} and Fact~\ref{thm:lattimore} prove the statement of Corollary~\ref{cor:lp_bayes} for bounded rewards.
\end{proof}

\begin{theorem}[Problem-dependent Local-Private Regret Bound (Theorem~\ref{thm:lb_log_local} in text)]
	For any asymptotically consistent bandit algorithm $\pol$, an environment $\model$ with optimal reward distribution $f^*$, and an $\epsilon$-locally private reward generation mechanism, the expected cumulative regret
	\begin{equation*}
	\liminf_{T \rightarrow \infty} \frac{\reg(\pol, \model, T)}{\log T} \geq \sum_{a \neq a^*} \frac{\Delta_a}{2\min\lbrace4,e^{2\epsilon}\rbrace(e^{\epsilon}-1)^2 \kldiv{f_a}{f^*}}.
	\end{equation*}
\end{theorem}
\begin{remark}
Step 1 and Step 2 of this proof are similar in essence as that of Theorem~\ref{thm:lb_local}. Step 3 differs from the fact that rather than substituting the KL-divergence and the expected number of draws using the problem-independent terms, we keep the problem dependent terms. This leads to a problem-dependent bound in Step 4.
\end{remark}
\begin{proof}\ \\
	
	\textbf{Step 1:} Fix two environments $\model_1$ and $\model_2$ such that $\model_1$ contains $K$ reward distributions $\lbrace f_1, \ldots, f_K\rbrace$ and $\model_2$ contains $K-1$ same reward distributions while the reward distribution $i$-th arm $f_i$ is replaced by $f'_i$ such that $\kldiv{f_i}{f'_i} \leq \kldiv{f_i}{f^*} + \delta$ for some $\delta > 0$ and $\mu(f'_i) > \mu(f^*)$. Here, $f^*$ represents the privatised reward distribution obtained for the optimal arm $a^*$.
	After imposing the $\epsilon$-local private mechanism, we obtain privatised reward distribution $\lbrace g_1, \ldots, g_i, \ldots, g_K\rbrace$ and $\lbrace g_1, \ldots, g'_i, \ldots, g_K\rbrace$.
	Let us denote the expected privatised rewards corresponding to the distributions as $\lbrace \mu_1, \ldots, \mu_K\rbrace$ and $\mu'_i$.
	
	\textbf{Step 2:}  We get the lower bounds of expected cumulative regret for a policy $\pol$, and the environments $\model_1$ and $\model_2$ as follows:
	\begin{align*}
	\reg(\pol, \model_1, T) &\geq \prob_{\pol\model_1}\left(N_1(T) \leq T/2\right) \frac{T}{2} (\mu_i - \mu^*),\\
	\reg(\pol, \model_2, T) &> \prob_{\pol\model_2}\left(N_1(T) > T/2\right) \frac{T}{2} (\mu'_i - \mu^*).
	\end{align*}
	
	Let us denote the event $N_1(T) \leq T/2$ as $E$.
	Hence, we get
	\begin{align*}
	\reg(\pol, \model_1, T) + \reg(\pol, \model_2, T)
	&\geq \frac{T}{2} \left(\prob_{\pol\model_1}(E) (\mu_i - \mu^*) + \prob_{\pol\model_2}(E^C)(\mu'_i - \mu^*)\right)\\
	&\geq  \frac{T}{2} \left(\prob_{\pol\model_1}(E) + \prob_{\pol\model_2}(E^C)\right)\min \lbrace (\mu_i - \mu^*), (\mu'_i - \mu^*)\rbrace\\ 
	&\underset{(a)}{\geq} \frac{T}{4} \exp(-\kldiv{\prob_{\pol\model_1}}{\prob_{\pol\model_2}})\min \lbrace (\mu_i - \mu^*), (\mu'_i - \mu^*)\rbrace.
	\end{align*}
	We obtain the inequality (a) from the Lemma 2.1 in~\citep{bretagnolle1979estimation} as mentioned in the proof of Theorem~\ref{thm:lb_local}.
	
	\textbf{Step 3:} From Lemma~\ref{thm:kldecomp_local}, we get 
	\begin{align*}
	\kldiv{\prob_{\pol\model_1}}{\prob_{\pol\model_2}}
	&= \expect_{\pol\model_1}[N_i(T)]\kldiv{g_i}{g'_i}\\
	&\underset{(b)}{\leq} 2\min\lbrace4,e^{2\epsilon}\rbrace (e^{\epsilon}-1)^2 \mathbb{E}_{\pol\model_1}\left[N_i(T)\right]\kldiv{f_i}{f'_i}\\
	&\leq 2\min\lbrace4,e^{2\epsilon}\rbrace (e^{\epsilon}-1)^2 \mathbb{E}_{\pol\model_1}\left[N_i(T)\right] (\kldiv{f_i}{f^*} + \delta)
	\end{align*}
	Step (b) is obtained from Lemma~\ref{thm:kldecomp_local}.
	
	Hence, we get the regret bound to be
	\begin{align*}
	&\reg(\pol, \model_1, T) + \reg(\pol, \model_2, T)\\
	\geq &\frac{T}{4}\min \lbrace (\mu_i - \mu^*), (\mu'_i - \mu^*)\rbrace\exp\left[-2\min\lbrace4,e^{2\epsilon}\rbrace (e^{\epsilon}-1)^2 \mathbb{E}_{\pol\model_1}\left[N_i(T)\right] (\kldiv{f_i}{f^*} + \delta)\right].
	\end{align*}
	Taking logarithm on both sides and simplifying, we get
	\begin{align*}
	\quad\frac{\expect_{\pol\model_1}\left[N_i(T)\right]}{\log T}
	&\geq \frac{1}{2 \min\lbrace4,e^{2\epsilon}\rbrace (e^{\epsilon}-1)^2 (\kldiv{f_i}{f^*} + \delta)}\frac{\log\left(\frac{T \min \lbrace (\mu_i - \mu^*), (\mu'_i - \mu^*)\rbrace}{4(\reg(\pol, \model_1, T) +\reg(\pol, \model_2, T))}\right)}{\log T}\\
	&\geq \dfrac{1+\frac{\log\left(0.25 \min \lbrace (\mu_i - \mu^*), (\mu'_i - \mu^*)\rbrace\right)}{\log T}-\frac{\log\left(\reg(\pol, \model_1, T) + \reg(\pol, \model_2, T)\right)}{\log T}}{2 \min\lbrace4,e^{2\epsilon}\rbrace (e^{\epsilon}-1)^2 (\kldiv{f_i}{f^*} + \delta)}.
	\end{align*}
	
	\textbf{Step 4:} We obtain the asymptotic lower bound by taking limit inferior on both sides
	\begin{align*}
	&\quad\liminf_{T \rightarrow \infty} \frac{\expect_{\pol\model_1}\left[N_i(T)\right]}{\log T}\\
	&\geq \liminf_{T \rightarrow \infty} ~~(2 \min\lbrace4,e^{2\epsilon}\rbrace (e^{\epsilon}-1)^2 (\kldiv{f_i}{f^*} + \delta))^{-1}\\
	&\qquad\qquad~~[1+\frac{\log\left(0.25 \min \lbrace (\mu_i - \mu^*), (\mu'_i - \mu^*)\rbrace\right)}{\log T} -\frac{\log\left(\reg(\pol, \model_1, T) + \reg(\pol, \model_2, T)\right)}{\log T}]\\
	&= (2 \min\lbrace 4,e^{2\epsilon}\rbrace (e^{\epsilon}-1)^2 (\kldiv{f_i}{f^*} + \delta))^{-1}\\
	&\qquad[1+ \liminf\limits_{T \rightarrow \infty} \frac{\log\left(0.25 \min \lbrace (\mu_i - \mu^*), (\mu'_i - \mu^*)\rbrace\right)}{\log T} - \limsup\limits_{T \rightarrow \infty} \frac{\log\left(\reg(\pol, \model_1, T) + \reg(\pol, \model_2, T)\right)}{\log T}]\\
	&\underset{(b)}{\geq} \dfrac{1 - p}{2 \min\lbrace 4,e^{2\epsilon}\rbrace (e^{\epsilon}-1)^2 (\kldiv{f_i}{f^*} + \delta)}\\
	&\underset{(c)}{\geq} \dfrac{1}{2 \min\lbrace 4,e^{2\epsilon}\rbrace (e^{\epsilon}-1)^2 \kldiv{f_i}{f^*}}.
	\end{align*}
	We obtain inequality (b) because:
	
	1. The first limit contains a constant in the numerator. Thus the limit goes to $0$ as $T \rightarrow \infty$.
	
	2. In order to obtain the other limit. We use the asymptotic consistency property of $\pi$. Since $\pi$ is assumed to be asymptotically consistent, $\reg(\pol, \model_1, T) + \reg(\pol, \model_2, T) \leq C T^p$ for some constant $p \in (0,1)$ and large enough time horizon $T$. Thus, $\limsup\limits_{T \rightarrow \infty} \frac{\log\left(\reg(\pol, \model_1, T) + \reg(\pol, \model_2, T)\right)}{\log T} \leq \limsup\limits_{T \rightarrow \infty} \frac{p \log T + \log C}{\log T} = p$.
	
	We obtain the inequality (c) in two steps: first, $1\geq p > 0$ and second, we take the limit as $\delta \rightarrow 0$.
	
	\textbf{Step 5:} Using the definition of regret and the resulting inequality of Step 4, we obtain
	\begin{align*}
	\liminf_{T \rightarrow \infty} \frac{\reg(\pol, \model_1, T)}{\log T}
	&= \liminf_{T \rightarrow \infty} \sum_{a\neq a^*} \frac{\expect_{\pol\model_1}\left[N_a(T)\right] (\mu_a - \mu^*)}{\log T}\\
	&\geq \sum_{a\neq a^*} \dfrac{(\mu_a - \mu^*)}{2 \min\lbrace 4,e^{2\epsilon}\rbrace (e^{\epsilon}-1)^2 \kldiv{f_a}{f^*}}\\
	&= \sum_{a\neq a^*} \dfrac{\Delta_a}{2 \min\lbrace 4,e^{2\epsilon}\rbrace (e^{\epsilon}-1)^2 \kldiv{f_a}{f^*}}.
	\end{align*}
\end{proof}
\textbf{Summary:} \textit{These results establish that the lower bounds for both minimax and Bayesian minimax regret  degrade by a multiplicative factor $\dfrac{1}{\min\lbrace 2,e^{\epsilon}\rbrace (e^{\epsilon}-1)}$ whereas the problem-dependent lower bound degrades by a multiplicative factor $\dfrac{1}{2\min\lbrace 4,e^{2\epsilon}\rbrace (e^{\epsilon}-1)^2}$ in an $\epsilon$-LDP bandit problem.}

\ifdefined \workshop
\
\else		
\subsection{Proofs for Instantaneous Differential Privacy (IDP)}
\begin{lemmarep}[Instantaneous DP KL-divergence Decomposition]
	For an $\epsilon$-IDP bandit algorithm $\pol$ satisfying $l(T) \leq \expect_{\pol\model}[N_a(T)]$ for any arm $a$,
	and two environments $\model_1$ and $\model_2$,
	\begin{align*}
	\begin{split}
	\kldiv{\prob^T_{\pol{\model_1}}}{\prob^T_{\pol{\model_2}}} &\leq 2\epsilon(e^{2\epsilon}-1)\frac{1-2e^{-\frac{T}{l(T)}}}{1 - e^{-\frac{T}{l(T)}}}+ \sum_{a=1}^K \mathbb{E}_{\pol \model_1}\left[N_a(T)\right](\kldiv{f_a \in \model_1}{f_a \in \model_2}).
	\end{split}
	\end{align*}
\end{lemmarep}
\begin{proof}
	The second term in the generic KL-divergence decomposition (Lemma~\ref{thm:kldecomp}) remains the same. 
	
	The first term is bounded as follows:
	\begin{align*}
	\sum_{t=1}^T \kldiv{\pol(A_t|\mathcal{H}^t, \model_1)}{\pol(A_t|\mathcal{H}^t, \model_2)}
	&= \sum_{t=1}^T \abs{\kldiv{\pol(A_t|\mathcal{H}^t, \model_1)}{\pol(A_t|\mathcal{H}^t, \model_2)}}\\
	&\leq \sum_{t=1}^T\max_{A_t,\mathcal{H}_t}\abs{\pol(A_t|\mathcal{H}^t, \model_1)}\abs{\log\frac{\pol(A_t|\mathcal{H}^t, \model_1)}{\pol(A_t|\mathcal{H}^t, \model_2)}}\\
	&\leq 2\epsilon(e^{2\epsilon}-1) \sum_{t=1}^T \max_{A_t,\mathcal{H}_t}  \pol(A_t|\mathcal{H}^t, \model_1)\\
	&\leq 2\epsilon(e^{2\epsilon}-1) \max_{A_t,\mathcal{H}_t}  \sum_{t=1}^T \pol(A_t|\mathcal{H}^t, \model_1)\\
	&\leq 2\epsilon(e^{2\epsilon}-1)\frac{1-2e^{-\frac{T}{l(T)}}}{1 - e^{-\frac{T}{l(T)}}}
	\end{align*}
	for $\expect_{\pol \model_1}[N_a(T)] \geq l(T)$ for all arms $a$.
\end{proof}
\begin{theoremrep}[Instantaneous DP Minimax Regret Bound]
	Given a finite privacy level $\epsilon \leq a/2$, and a time horizon $T \geq h(K,\epsilon)$, then for any environment with finite variance, any algorithm $\pol$ that is $\epsilon$-IDP satisfies
	\begin{equation*}
	\reg_{\mathrm{minimax}}(T) \geq c(a)
	\sqrt{\frac{(K-1)T}{2\epsilon(e^{2\epsilon}-1)}}.
	\end{equation*}
\end{theoremrep}
\begin{proof}\ \\
	Repeat the \textbf{Steps 1} and \textbf{2} described in the proof sketch of Theorem~\ref{thm:lb_local}.
	
	\textbf{Step 3:} Use Lemma~\ref{thm:kldecomp_global} for KL-divergence decomposition under differential privacy to obtain
	\begin{align*}
	\kldiv{\prob_{\pol\model_1}}{\prob_{\pol\model_2}} \leq 2\epsilon(e^{2\epsilon}-1) + \frac{2T\Delta^2}{K-1}.    
	\end{align*}
	Hence, we get the regret bound to be
	\begin{align*}
	\max\lbrace \reg_{\model_1}(\pol, T), \reg_{\model_2}(\pol, T) \rbrace
	&\geq \frac{1}{2}\left(\reg_{\model_1}(\pol, T) + \reg_{\model_2}(\pol, T)\right)\\
	&\geq \frac{T\Delta}{4}\exp\left[-2\epsilon(e^{2\epsilon}-1) - \frac{2T\Delta^2}{K-1}\right].
	\end{align*}
	
	\textbf{Step 4:} Let us choose the optimality gap $\Delta$ to be $$\Delta = \sqrt{\frac{(K-1)C(\epsilon)}{4\epsilon(e^{2\epsilon}-1) T}} \leq \frac{1}{2}.$$ 
	Here, we choose $$C(\epsilon) = - 2\epsilon(e^{2\epsilon}-1) W\left(-\frac{e^{-\delta(a)+2\epsilon(e^{2\epsilon}-1)}}{2\epsilon(e^{2\epsilon}-1)}\right),$$ where $W$ is the Lambert function or the product-log function and $\delta(a)$ is a function of $a$ such that $2\epsilon(e^{2\epsilon}-1) \leq \delta(a)$ for $\epsilon \leq a$. In the given range of $\epsilon$, $C(\epsilon) \leq 1$. 
	
	$\Delta$ being less than $\frac{1}{2}$ holds for any for $T \geq \frac{2(K-1)C(\epsilon)}{4\epsilon(e^{2\epsilon}-1)} \defn h(K,\epsilon)$. Hence, by combining the results of Step 3 and the upper bound on privacy level $\epsilon \leq a$ and $T \geq h(K,\epsilon)$, we obtain:
	\begin{equation}
	\reg_{minimax}(T) \geq \frac{e^{\delta(a)}}{4\sqrt{2}} \sqrt{\frac{(K-1)T}{\min\lbrace 2,e^{\epsilon}\rbrace (e^{2\epsilon}-1)}}.
	\end{equation}
\end{proof}
\begin{corollaryrep}[Instantaneously Private Bayesian Minimax Regret Bound]
	Given a finite privacy level $\epsilon \in \real$, and a finite time horizon $T$, then for any environment with bounded reward $\br \in [0,1]^K$, any algorithm $\pol$ that is $\epsilon$-IDP satisfies
	\begin{equation*}
	\reg^*_{\mathrm{Bayes}}(T) \geq c(\epsilon)
	\sqrt{\frac{(K-1)T}{2\epsilon(e^{2\epsilon}-1)}}.
	\end{equation*}
\end{corollaryrep}
\begin{proof}
	Similar to the proof of Corollary~\ref{cor:lp_bayes}, the results of Theorem~\ref{thm:lb_global} and Fact~\ref{thm:lattimore} prove the statement of Corollary~\ref{cor:lb_bayes_global}.
\end{proof}
\fi

\ifdefined \workshop
\subsection{Lower Bounds for Differential Privacy (DP)}
\else
\subsection{Proofs for Differential Privacy (DP)}
\fi
There are two possibilities under DP which basically takes same set of $T$ actions $a^T$ for two neighbouring set of generated rewards $\bx^T$ and $\bx'^T$ with a difference of $e^{\epsilon}$.

In the first case, the one differently generated reward at step $k$, $x'_{k, a}$ does not come from the chosen arm i.e. $a_k \neq a$. In this case, the bound becomes trivial and derivation goes back to the classical version.

The second case is more interesting one. In this case, the arm $a$ is chosen at step $k$ and the differently generated reward $x'_{k, a}$ becomes the observed reward $r'_k$.
We derive all the following bounds under this second scenario.
\begin{lemma}\label{lemma:hist_dp}
	If a policy $\pi$ satisfy $\epsilon$-DP for an environment $\model$ and $\ln \sup_{a, x_a,x'_a} \frac{\prob_{\model}(x_a)}{\prob_{\model}(x'_a)} \leq L \Delta$, we get
	\begin{equation}
		\prob_{\model \pol}(\Hist_T) \leq e^{\epsilon + L\Delta} \prob_{\model \pol}(\Hist'_T).
	\end{equation}
	Here, $L$ is the Lipschitz constant corresponding to the reward distributions in the environment and $\Delta \triangleq \max \rho(x_a, x'_a)$.
\end{lemma}
\begin{proof}
	History till horizon $T$ is a collection of actions and observed rewards $\Hist_T = (a^T, r^T) \triangleq  \lbrace a_1, \ldots, a_T, r_1, \ldots, r_k, \ldots, r_T\rbrace = \lbrace a_1, \ldots, a_T, x_{1, a_1}, \ldots, x_{k, a_k}, \ldots, x_{T, a_T}\rbrace$. Thus, we can define a neighbouring history for DP definition as $\Hist_T = (a^T, r'^T) \triangleq \lbrace a_1, \ldots, a_T, r_1, \ldots, r'_k, \ldots, r_T\rbrace$.
	$S$ and $S'$ denote sets of all $\bx^T$ and $\bx'^T$s for which we get $r^T$ and $r'^T$ as the rewards sequence given a sequence of actions $a^T$.
	We observe that $S \cup \{\bx_k(a_k) = r'_k\} = S' \cup \{\bx_k(a_k) = r_k\} = \real^{KT}$.
	So, the sets of generated rewards that lead to two neighbouring reward sequences comes from two subsets of $\real^{KT}$ with two specific punctures at $r'_k$ and $r_k$ respectively.
	With this underlying structure, we now derive the relation between the densities induced by two neighbouring histories.
	\begin{align*}
		\prob_{\model \pol}(\Hist_T)
		 &= \int_S \prob_{\pol}(\Hist_T| \bx^T) d\prob_{\model}(\bx^T) \\
		 &= \int_S \prob_{\pol}(a^T, r^T| \bx^T) d\prob_{\model}(\bx^T) \\
		 &= \int_S \prob_{\pol}(a^T | \bx^T) \prob_{\pol}(r^T | a^T, \bx^T) d\prob_{\model}(\bx^T) \\
		 &\underset{(a)}{=} \int_{S} \prob_{\pol}(a^T | \bx^T ) \prob(r'^T | a^T, \bx^T + \Delta_k) d\prob_{\model}(\bx^T) e^{L\Delta_k}\\
		 &\underset{(b)}{\leq} \int_{S} e^{\epsilon} \prob_{\pol}(a^T | \bx^T + \Delta ) \prob(r'^T | a^T, \bx^T + \Delta_k) d\prob_{\model}(\bx^T) \\
		 &\underset{(b)}{\leq} e^{\epsilon} \int_{S}  \prob_{\pol}(a^T | \bx^T + \Delta_k ) \prob(r'^T | a^T, \bx^T + \Delta_k)~e^{L\abs{\Delta_k}} d\prob_{\model}(\bx^T + \Delta_k) \\
		 &\underset{(d)}{=} e^{\epsilon+L\abs{\Delta_k}}  \int_{S'} \prob_{\pol}(a^T | \bx'^T) \prob(r'^T | a^T, \bx'^T) d\prob_{\model}(\bx'^T)\\
		 &\leq e^{\epsilon+L\Delta} \prob_{\model \pol} (\Hist'_T)
	\end{align*}
Here, $\Delta_k = r_k - r'_k \neq 0$ and $\Delta = \max |r_k - r'_k|$.
	Step (a) is obtained from the fact that the observed rewards $r_t$ and $r'_t$ are deterministic given the corresponding actions and generated rewards.
	
	Inequality in step (b) comes from the definition of $\epsilon$-DP i.e. $\prob_{\pol}(a^T | \bx^T) \leq e^{\epsilon} \prob_{\pol}(a^T | \bx'^T) = e^{\epsilon} \prob_{\pol}(a^T | \bx^T + \Delta_k)$.
	
	The second inequality in step (c) is obtained from the Lipschitzness assumption on the reward distributions of the given environment. Specifically, $\sup_{a} \frac{\prob_{\model}(x_a)}{\prob_{\model}(x'_a)} \leq e^{L |x_a - x'_a|}$, where L is the Lipschitz constant for the corresponding reward distributions.
	
	Finally, step (d) is change of variable in the integral. This shift is possible if and only if the joint probability density on generated reward $\bx$ is smooth and thus non-zero for a given environment with $K$ reward distributions. Thus, both $r_k$ and $r'_k$ has well-defined double sided limit and their corresponding measures under this density function $\prob_{\model}$ are zero\footnote{Note that $\prob_\model \triangleq \prod_{a=1}^K f_a$.}. 

	Hence, we get that $\prob_{\model \pol}(\Hist_T) \leq e^{\epsilon + L\Delta} \prob_{\model \pol}(\Hist'_T)$.
\end{proof}

This assumption on Lipschitz continuity of the reward distribution is used in several occasions in literature.
It does not hold if there are deterministic distributions, like Dirac delta on a single reward value, but for most of the distributions, such as exponential family.
For exponential family, if the sufficient statistics is smooth enough with Lipschitz constant less than or equal to $\frac{L}{\eta}$ for some natural parametrisation $\eta$, we naturally satisfy this assumption.
Thus, this assumption allows us to deal with a large family of reward distributions used in practice, including Gaussian, Beta, Exponential, Bernoulli etc., while allowing the continuation of analysis.

\begin{lemmarep}[Differentially Private KL-divergence Decomposition]\label{thm:kldecomp_DP}
	If the policy is $\epsilon$-local differentially private for both the environments $\model_1$ and $\model_2$, we get for two neighbouring histories $\Hist_T$ and $\Hist'_T$ 
	\begin{align*}
	\kldiv{\prob_{\pol{\model_1}}(\Hist_T)}{\prob_{\pol{\model_2}}(\Hist_T)} &\leq  e^{\epsilon + c} ( 2(\epsilon + c)
	+ \kldiv{\prob_{\pol{\model_1}}(\Hist'_T)}{\prob_{\pol{\model_2}}(\Hist'_T)} ).
	\end{align*}
	Here, $c=L\Delta$ where $L$ is the Lipschitz constant corresponding to the reward distributions of the environments and $\Delta$ is the difference in generated rewards.
\end{lemmarep}
\begin{proof}
The idea is bounding 
\[\frac{\prob_\pol(\Hist|\model_1)}{\prob_\pol(\Hist|\model_2)} = \frac{\prob_\pol(\Hist|\model_1)}{\prob_\pol(\Hist'|\model_1)} \frac{\prob_\pol(\Hist'|\model_1)}{\prob_\pol(\Hist'|\model_2)}
\frac{\prob_\pol(\Hist'|\model_2)}{\prob_\pol(\Hist|\model_2)} \leq e^{2(\epsilon + c)} \frac{\prob_\pol(\Hist'|\model_1)}{\prob_\pol(\Hist'|\model_2)}
\]
This will give us the bound on KL: 
\begin{align*}
	\kldiv{\prob_{\pol{\model_1}}(\Hist_T)}{\prob_{\pol{\model_2}}(\Hist_T)} &= \int d\prob_{\pol{\model_1}}(\Hist_T) \ln \left( \frac{\prob_{\pol{\model_1}}(\Hist_T)}{\prob_{\pol{\model_2}}(\Hist_T)} \right)\\
	&= \int d\prob_{\pol{\model_1}}(\Hist_T)\ln \frac{\prob_{\pol{\model_1}}(\Hist_T)}{\prob_{\pol{\model_2}}(\Hist_T)} \\
	&\leq \int \ln e^{2\epsilon + 2c}\frac{\prob_{\pol{\model_1}}(\Hist'_T)}{\prob_{\pol{\model_2}}(\Hist'_T)} d\prob_{\pol{\model_1}}(\Hist_T)\\
	&= 2(\epsilon+c)  + \int \ln \frac{\prob_{\pol{\model_1}}(\Hist'_T)}{\prob_{\pol{\model_2}}(\Hist'_T)} d\prob_{\pol{\model_1}}(\Hist'_T) \frac{d\prob_{\pol{\model_1}}(\Hist_T)}{d\prob_{\pol{\model_1}}(\Hist'_T)}\\
	&\leq 2(\epsilon + c) + e^{2(\epsilon + c)} \kldiv{\prob_{\pol{\model_1}}(\Hist'_T)}{\prob_{\pol{\model_2}}(\Hist'_T)}
\end{align*}
\end{proof}

\begin{theoremrep}[DP Minimax Regret Bound]
	Given a finite privacy level $\epsilon > 0$, and a time horizon $T \geq h(K,\epsilon)$, then for any environment with finite variance, any algorithm $\pol$ that is $\epsilon$-DP satisfies
	\begin{align*}
	\reg_{\mathrm{minimax}}(T) &\geq \frac{1}{8 e^{3(\epsilon+c)}} \sqrt{\frac{(K-1)T \ln(\epsilon+1) }{\epsilon^{(1+\frac{1}{\epsilon})}(\epsilon^2+1)^{\frac{1}{\epsilon}}}} 
	\end{align*}
	Here, $c=L\Delta$ where $L$ is the Lipschitz constant corresponding to the reward distributions of the environments and $\Delta$ is the difference in generated rewards.
\end{theoremrep}
\begin{proof}\ \\
	Repeat the \textbf{Steps 1} and \textbf{2} described in the proof sketch of Theorem~\ref{thm:lb_local}.
	
	\textbf{Step 3: }Use Lemma~\ref{thm:kldecomp_DP} for KL-divergence decomposition under differential privacy to obtain
		\begin{align*}
		\kldiv{\prob_{\pol{\model_1}}(\Hist_T)}{\prob_{\pol{\model_2}}(\Hist_T)} &\leq 2(\epsilon + c) + e^{2(\epsilon + c)} \kldiv{\prob_{\pol{\model_1}}(\Hist'_T)}{\prob_{\pol{\model_2}}(\Hist'_T)}\\
		&\leq 2(\epsilon + c) + e^{2(\epsilon + c)}\frac{2T\Delta^2}{K-1}
		\end{align*}
	Hence, we get the regret bound to be
	\begin{align*}
	\max\lbrace \reg_{\model_1}(\pol, T), \reg_{\model_2}(\pol, T) \rbrace
	&\geq \frac{1}{2}\left(\reg_{\model_1}(\pol, T) + \reg_{\model_2}(\pol, T)\right)\\
	&\geq \frac{T\Delta}{4}\exp\left[-2(\epsilon + c) - e^{2(\epsilon + c)} \frac{2T\Delta^2}{K-1}\right].
	\end{align*}
	
	\textbf{Step 4:} Let us choose the optimality gap $\Delta$ to be $$\Delta = \sqrt{\frac{(K-1) \ln(\epsilon^2+1) }{4T\epsilon e^{2(\epsilon + c)}} } \leq \frac{1}{2}.$$ 
		
	$\Delta$ being less than $\frac{1}{2}$ holds for any for $T \geq \frac{(K-1) \ln(\epsilon^2+1) }{\epsilon e^{2(\epsilon + c)}} \defn h(K,\epsilon)$. 
	Hence, by substituting this value of suboptimality gap in the result of Step 3, we obtain for $T \geq h(K,\epsilon)$:
	\begin{align*}
	\reg_{\mathrm{minimax}}(T) &\geq \frac{1}{8 e^{\epsilon+c}} \sqrt{\frac{(K-1)T \ln(\epsilon^2+1) }{\epsilon}} \exp\left[ -2(\epsilon + c ) - \frac{\ln(\epsilon^2+1) }{2\epsilon} \right]\\
	&\geq \frac{1}{8 e^{3(\epsilon+c)}} \sqrt{\frac{(K-1)T \ln(\epsilon+1) }{\epsilon^{(1+\frac{1}{\epsilon})}(\epsilon^2+1)^{\frac{1}{\epsilon}}}} 
	\end{align*}
\end{proof}
\begin{corollaryrep}[DP Bayesian Minimax Regret Bound]
	Given a finite privacy level $\epsilon \in \real$, and a finite time horizon $T \geq h(K, \epsilon)$, then for any environment with bounded reward $\br \in [0,1]^K$, any algorithm $\pol$ that is $\epsilon$-IDP satisfies
		\begin{align*}
		\reg_{\mathrm{minimax}}(T) &\geq \frac{1}{8 e^{3(\epsilon+c)}} \sqrt{\frac{(K-1)T \ln(\epsilon+1) }{\epsilon^{(1+\frac{1}{\epsilon})}(\epsilon^2+1)^{\frac{1}{\epsilon}}}} 
		\end{align*}
		Here, $c=L\Delta$ where $L$ is the Lipschitz constant corresponding to the reward distributions of the environments and $\Delta$ is the difference in generated rewards.
\end{corollaryrep}
\begin{proof}
	Similar to the proof of Corollary~\ref{cor:lp_bayes}, the results of Theorem~\ref{thm:lb_dp} and Fact~\ref{thm:lattimore} prove the statement of Corollary~\ref{cor:lb_bayes_dp}.
\end{proof}
\textbf{Summary:} \textit{These results establish that the lower bounds for both minimax and Bayesian minimax regret  degrade by a multiplicative factor $\sqrt{\dfrac{\ln(\epsilon^2+1) }{e^{6\epsilon}\epsilon^{(1+\frac{1}{\epsilon})}(\epsilon+B)^{\frac{1}{\epsilon}}}}$ for an $\epsilon$-DP bandit algorithm.
The limit of the lower bound goes to infinity as $\epsilon \rightarrow 0$ i.e. the policy becomes completely random. In such a situation, the regret for bandits gets capped by the order of $T$. As $\epsilon$ increases, the privacy dependent factor in lower bound also decreases monotonically.}

\end{document}